%% file: main.tex
\documentclass[dvipsnames,format=sigconf,anonymous=false,review=false]{acmart}
\AtBeginDocument{%
  }

\copyrightyear{2026}
\acmYear{2026}
\setcopyright{cc}
\setcctype{by}
\acmConference[GECCO '26]{Genetic and Evolutionary Computation Conference}{July 13--17, 2026}{San Jose, Costa Rica}
\acmBooktitle{Genetic and Evolutionary Computation Conference (GECCO '26), July 13--17, 2026, San Jose, Costa Rica}
\acmDOI{10.1145/3795095.3805066}
\acmISBN{979-8-4007-2487-9/2026/07}

\input{_used_packages.tex}
\input{_macros.tex}

\begin{document}

\title{Anytime Analysis on \BinVal: Adaptive Parameters Help}

\author{Timo Kötzing}
\email{Timo.Koetzing@hpi.de}
\orcid{0000-0002-1028-5228}
\affiliation{%
  \institution{Hasso Plattner Institute\\ University of Potsdam}
  \city{Potsdam}
  \country{Germany}
}

\author{Jurek Sander}
\email{Jurek.Sander@hpi.de}
\orcid{0009-0004-6707-6592}
\affiliation{%
  \institution{Hasso Plattner Institute\\ University of Potsdam}
  \city{Potsdam}
  \country{Germany}
}

\renewcommand{\shortauthors}{xxx}

\begin{abstract}
  While most theoretical run time analyses of discrete randomized search heuristics provide bounds on the expected number of evaluations to find the global optimum, we consider the anytime performance of evolutionary and estimation-of-distribution algorithms.
  For this purpose, we analyze the fixed-target run time of various algorithms using \BinVal as fitness function and bound the run time to optimize the most significant $k \in o(n)$ bits of a bit string with length $n$.
  We analyze the run times such that they hold not only for a fixed $k$, but simultaneously for all $k \in o(n)$.
    
  For the standard \EA with fixed mutation rate $1/n$, we show that the fixed-target run time for all $k \in o(n)$ is in $\Theta(n \log k)$.
  Using an EDA instead, we get an expected number of evaluations of $\On(k \log n)$ for the sig-cGA.
  Replacing in the standard \EA the fixed mutation rate with a self-adjusting rate, we show that the fixed-target run time for $k \in o(n)$ and a constant $\varepsilon >0$ arbitrarily close to zero is in~$\On\Brackets{k^{1+\varepsilon}}$ for this algorithm.
  In particular, this run time is independent of $n$, holds simultaneously for all $k \in o(n)$, and is close to the run time of $\Theta(k \log k)$ for the \EA with the best fixed mutation rate if $k$ is known.
\end{abstract}



\maketitle

\input{sect_Introduction}
\input{sect_Preliminaries}
\input{sect_EFHT_EA}
\input{sect_EFHT_sig-cGA}
\input{sect_EFHT_EA_AdjustingMR}
\input{sect_EA_Self-Adjusting-MR}
\input{sect_Experiments}

\input{sect_Conclusion}
\clearpage
\bibliographystyle{ACM-Reference-Format}
\bibliography{_ref.bib}

\section*{Acknowledgements}
This research was (partially) funded by the HPI Research School on Foundations of AI (FAI).

\clearpage
\newpage
\appendix
\input{sect_ExperimentsAppendix}

\end{document}

%% file: _used_packages.tex
\usepackage{graphicx}
\usepackage[vlined,linesnumbered, ruled]{algorithm2e}
\usepackage[createShortEnv]{proof-at-the-end}
\usepackage{cleveref}

%% file: _macros.tex
\newcommand{\N}{\mathbb{N}}
\newcommand{\natnum}{\mathbb{N}}
\newcommand{\R}{\mathbb{R}}

\newcommand{\On}{\mathcal{O}}

\newcommand{\ExMath}[1]{\mathrm{E}\left[#1\right]}
\newcommand{\PrMath}[1]{\mathrm{Pr}\left[#1\right]}
\newcommand{\OMath}[1]{\mathcal{O}\left(#1\right)}

\SetKwRepeat{Do}{do}{while}%
\SetKwIF{WithProb}{ElseWithProb}{Else}{with probability}{do}{else with probability}{else}{end}
\newcommand{\assign}{\leftarrow}

\definecolor{mygreen}{RGB}{1, 150, 122}

\providecommand{\ignore}[1]{}

\newcommand{\OneMax}{\textsc{One\-Max}\xspace}
\newcommand{\BinVal}{\textsc{Bin\-Val}\xspace}
\newcommand{\LeadingOnes}
{\textsc{Leading\-Ones}\xspace}
\newcommand{\EA}{(1+1)~EA\xspace}

\renewcommand{\epsilon}{\varepsilon}

\newcommand{\Brackets}[1]{\left ( #1 \right )}
\newcommand{\SqBrackets}[1]{\left [ #1 \right ]}
\newcommand{\CurlyBrackets}[1]{\left \{ #1 \right \}}
\newcommand{\DoubleVert}[1]{\left \lvert \left \lvert #1 \right \rvert \right \rvert}


%% file: sect_Introduction.tex
\section{Introduction}
\label{sect:intro}
Randomized search heuristics, such as evolutionary algorithms, iteratively generate and evaluate solutions to guide the search towards the global optimum while treating the underlying fitness function as a black box.
The iterative search approach leads to the \emph{anytime} property of evolutionary algorithms, as they produce a solution at any time during their optimization process, and the quality of the solution improves over time \cite{buzdalovdoerrfixedtarget2020,rowe_evo_anytime_2025}.
Utilizing this property and stopping the algorithm early is especially beneficial in settings where the most significant progress occurs at the beginning of the optimization, and afterwards the solution is only fine-tuned. In particular, for difficult optimization problems, finding the global optimum is typically neither feasible nor detectable. In this case, \emph{stopping criteria} can be used to determine that the algorithm obtained a sufficient quality. 
Quantifying the \emph{anytime} performance in this setting has thus also practical relevance for understanding stopping criteria, so that the algorithm uses significantly fewer fitness evaluations than it would require to obtain the optimal solution, but for which a sufficiently good solution can be obtained \cite{buzdalovdoerrfixedtarget2020, Zilberstein_1996,RoweTerminationCrit}.

While most theoretical analyses of randomized search heuristics focus on the expected time to find the global optimum, there are two less studied approaches: the \emph{fixed-budget} and \emph{fixed-target} analyses~\cite{buzdalovdoerrfixedtarget2020}.
For the \emph{fixed-budget analysis}, a number of fitness evaluations (``budget'') is fixed, and the expected quality of the best solution found using this budget is analyzed \cite{jansen_zarges_fixed_budget_2012,buzdalovdoerrfixedtarget2020}.
In contrast, for the \emph{fixed-target analysis}, a quality threshold is defined, and the expected number of fitness evaluations required by the algorithms to find a solution that meets this threshold is studied \cite{buzdalovdoerrfixedtarget2020, rowe_evo_anytime_2025}.

In this work, we consider the search space $\{0,1\}^n$. We are interested in the anytime performance of randomized search heuristics: for a family of fitness targets $(x_k)_k$, we aim to bound the expected number of fitness evaluations until the algorithm obtains a fitness better than $x_k$, in dependence on $n$ and $k$. Note that, while the algorithm has access to $n$ (and can make decisions based on $n$, for example using parameters in dependence on $n$), it has no access to~$k$ and the performance is measured for all $k$ ``simultaneously'' in this sense.\footnote{This is in contrast with algorithms that are given an optimization target and can optimize their behavior to this single goal.}

We analyze the performance on the ``binary value'' function \BinVal. As the name suggests, this fitness function assigns each bit string of fixed length $n$ its binary value:
\begin{align}
    \BinVal\colon \{0,1\}^n \rightarrow \mathbb{R}, x \rightarrow \sum_{i=1}^n 2^{n-i}x_i. \label{eq:binval}
\end{align}
Intuitively, this linear function weights the bits of the given bit string exponentially, such that each bit has a larger weight than all less significant bits combined.
Therefore, it models a scenario in which the anytime performance is relevant, as optimizing the most significant bits is sufficient to obtain a solution close to the optimum. Thus, as anytime targets we consider $\left(\sum_{i=1}^k 2^{n-i}\right)_k$, which corresponds exactly to setting the first $k$ bits correctly and analyzing the time to find the $k \in o(n)$ most significant bits.

We start by considering the well-studied \EA (see Section~\ref{sect:prelim}) in this setting. 
Since, with \BinVal as fitness function, less significant bits do not influence the optimization of the most relevant $k$ bits, optimizing the first $k$ bits of a $n$ bit problem is equivalent to the problem of finding the optimum of a $k$ bit problem for the \EA with fixed mutation rate.
Thus, it is possible to use the tight bounds of Witt in~\cite{WITT_2013} for the anytime performance of the algorithm, as displayed in the first three lines of Table~\ref{tab:expRuntimeResults}.
While the work~\cite{buzdalovdoerr_ftrwthresamp_2019} gives, directly in the anytime setting, tight lower bounds for a variant of the \EA, the upper bound gives the weak $O(n \log n)$ for a mutation rate of $1/n$. In Section~\ref{sect:anytime_basic_ea}, we provide proofs for simple and tight bounds for the \EA for the sake of completeness.
As can be seen, the algorithm with a fixed mutation rate of $1/n$ has a performance of $\Theta(n \log k)$ to optimize the first $k$ bits. Especially for very small targets $k$, this is a very long time (linear in $n$), and we are interested in algorithms where the dependence on $n$ is reduced or entirely removed.
Note that the desire is similar to that in the setting of ``unknown solution length'' studied in~\cite{cathabard_lehre_yao_unknown_solution_lenghts_2011,doerr_kötzing_unknown_sol_length_2015}, but different in that we have not just one unknown target.
In~\cite{EINARSSON2019150}, the authors present the ``hidden subset problem'', which is formally equivalent to the anytime performance problem for \BinVal.
Einarsson et al. study different ways of choosing the mutation rate of the \EA for this setting and prove that there is a scheduled sequence of mutation rates which achieves an anytime performance of $\Theta(k \log k)$ on \BinVal.
The same applies to an adaptive variant of the algorithm.

As a related result for the fixed-target analysis, using the fitness levels in \cite{sudholt_lower_bounds_running_time_2013}, Lengler and Spooner provide in \cite{lengler_spooner_fixed_budget_linear_functions_2015} a lower bound for the expected fixed-target results of the fitness function \OneMax.
The corresponding upper bounds for different variants of the \EA are given in \cite{pinto_doerr_more_practice_aware_2018}.

For the fixed-target analysis of linear functions using genetic and estimation-of-distribution algorithms (EDAs), Witt showed in~\cite{witt_domino_2018} that the \emph{domino convergence} effect holds for these algorithms and leads to an earlier optimization of bits of large weight than bits of low weight in the linear function.

In this paper, we study the anytime performance of various randomized search heuristics on \BinVal, with a focus on variants of the \EA.
The results of our analysis are summarized in Table~\ref{tab:expRuntimeResults}, which shows the expected number of fitness evaluations to optimize the $k \in o(n)$ most significant bits simultaneously for all $k$.
Our results significantly extend the existing fixed-target analysis for \BinVal and are structured as follows.

After studying the \EA in Section~\ref{sect:anytime_basic_ea}, in Section~\ref{sect:sig-cGA} we consider EDAs.
Among these algorithms, the sig-cGA \cite{doerr2018significance} appears to be particularly well-suited for this problem.
We extend the theoretical analysis given in \cite{doerr2018significance}, and get a fixed target run time of~$\Theta(k \log n)$ for the sig-cGA on  \BinVal.
While its dependence on $n$ is exponentially smaller than that of the \EA, it still depends on $n$, which has an especially significant impact for large $n$.

In order to overcome the dependency on $n$, we introduce in Section~\ref{sect:ea-adjusting} an idealized \EA with adjusting mutation rate, which iteratively sets its mutation rate such that it is optimal to optimize the next block of not yet optimized most significant bits.
The algorithm definition does not depend on a fixed target $k$, but requires the number of already optimized leading bits of the bit string in every iteration.
We prove that the algorithm achieves the fixed-target run time $\Theta(k\log k)$ for all $k$.

To generalize the \EA with adjusting mutation rate and its theoretical analysis further, we define, similar to \cite{doerr2019self}, the \EA with self-adjusting mutation rate that updates the mutation rate depending on the acceptance or rejection of the generated mutant in every iteration and does not require any knowledge about~$k$ or the current best solution.
For the self-adjusting \EA, we upper bound in Section~\ref{sect:ea-self-adj} the expected run time by $\OMath{k^{1+\varepsilon}}$ to optimize the~$k\in o(n)$ most significant bits with $\varepsilon:=1/\log(1/\gamma)$ with $\gamma$ as a constant and where $\varepsilon> 0$ can be chosen arbitrarily close to zero.

In Section~\ref{sect:experiments}, we compare experimentally the anytime performance of the standard \EA with its two studied variants.

As the main technique for our run time proofs of the fixed-target analyses, we use the drift theory, originally introduced by He and Yao in \cite{HE_yao_drift_2001,he_yao_drift_study_2004}.
Since we use \BinVal as the fitness function of the analyzed randomized search heuristics, we must consider a few properties that make the analysis more difficult and complex than analyses of algorithms optimizing \LeadingOnes or \OneMax.
In comparison to \LeadingOnes, for \BinVal, it is not necessary that the optimization follows strictly from the most significant to the least significant bit, and every not optimized bit can be responsible for the next progress.
But since the bit of the leftmost flip has a larger weight than all the less significant bits combined, the total number of optimized bits can be reduced even when the generated mutant gets accepted.
In comparison to the analysis of \OneMax, this requires a more advanced analysis and specially designed potential functions, especially for larger mutation rates \cite{WITT_2013, droste2002analysis1p1ea,doerr_goldberg_adaptive_drift_analysis_2010}.
In \cite{doerr2019self}, Doerr et al.\ analyze the optimization and adjustment of the mutation rate of the \EA with self-adjusting mutation rate separately.
In contrast, we design for the fixed-target analysis of the \EA with self-adjusting mutation rate on \BinVal, using a similar idea to \cite{doerr_koetzing_static_and_self_adjusting_2018}, a combined potential function that results in handling the optimization of the bit string and adjustment of the mutation rate simultaneously.
This makes the entire analysis more complex, but yields a concise result that shows the fixed-target run time to be independent of $n$ for all $k \in o(n)$.

\begin{figure*}[htp]
    \bgroup
    \def\arraystretch{1.5}
    \setlength{\tabcolsep}{10pt}
    \begin{tabular}{l|llll}
    \begin{tabular}[c]{@{}l@{}}Randomized search heuristic\end{tabular} & Lower Bound & & Upper Bound &  \\ \hline
    \EA with $\chi = 1 / n$ & $\Omega\Brackets{n \log k}$ & \cite[Cor.~6.6]{WITT_2013} & $\On\Brackets{n \log k}$ & \cite[Cor.~4.2]{WITT_2013}\\
    \EA with $\chi = c / n$, $c > 0$, $k \leq n/c$ & $\Omega\Brackets{\frac{n\log k}{c}}$ & \cite[Cor.~6.6]{WITT_2013} & $\On\Brackets{\frac{n\log k}{c}}$ &\cite[Cor.~4.2]{WITT_2013}\\
    \EA with $0 < \chi < 1/2$, $k \leq 1/\chi$ & $\Omega\Brackets{\frac{\log k}{\chi}}$ & \cite[Cor.~6.6]{WITT_2013} & $\On\Brackets{\frac{\log k}{\chi}}$ &\cite[Cor.~4.2]{WITT_2013}\\
    sig-cGA & $\Omega\Brackets{k\log n}$ & [Lem.~\ref{lem:sig-cGA_spec_anytime_analysis:lower-bound}] & $\On(k \log n)$ &[Cor.~\ref{cor:sig-cGA_spec_anytime_analysis}]\\
    sig-cGA with $s(\epsilon, \mu) = \epsilon\max\CurlyBrackets{\sqrt{\mu\ln \tilde{k}},\ln \tilde{k}}$, $k \leq \tilde{k}$ & $\Omega\Brackets{k\log \tilde{k}}$ & [Cor.~\ref{cor:sig-cGA_spec_anytime_analysis-with-tilde-k:lower-bound}] & $\On\Brackets{k \log \tilde{k}}$ &[Cor.~\ref{cor:sig-cGA_spec_anytime_analysis-with-tilde-k}]\\
    \EA with adjusting MR & $\Omega(k \log k)$ & [Cor.~\ref{cor:ea-with-adj-mut-rate-run-time-low-bound}] & $\On(k \log k)$ & [Thm.~\ref{thm:ea-with-adj-mut-rate-run-time}]\\
    \EA with self-adjusting MR, $\varepsilon:=1/\log(1/\gamma)$ & & & $\On\Brackets{k^{1+\varepsilon}}$& [Thm.~\ref{theo_runtime_self_adj_mr}]
    \end{tabular}
    \captionsetup{type=table}
    \caption{Bounds of the fixed-target run time for the most significant $k \in o(n)$ bits simultaneously for different randomized search heuristics on the \BinVal function. All Landau notations depend only on the bit length $n\in \N$ and the number $k$ of the most significant bits that should be optimized. $\gamma$ is a fixed constant and $\varepsilon>0$ can be chosen arbitrarily close to zero.}
    \label{tab:expRuntimeResults}
    \Description[Overview of fixed-target run time results.]{Overview of the fixed-target run time results of all analyzed \EA variants and the sig-cGA.}
    \egroup
\end{figure*}

%% file: sect_Preliminaries.tex
\section{Preliminaries}
\label{sect:prelim}

As preliminaries, we need the definition of the standard \EA with its mutation operator, the fitness function \BinVal, and the drift theorems that we use in our analyses.

We define the mutation operator for the \EA as follows.

\begin{definition}
    For a given mutation rate $0 < \chi < 1$ and bit string $x \in \CurlyBrackets{0,1}^n$ with length $n \in \natnum$, the mutation operator $\textsc{Mutate}(x,\chi)$ flips each of the $n$ bits independently with probability $\chi$ and returns the ``so-generated'' mutant $y \in \CurlyBrackets{0,1}^n$ of~$x$.
    \label{def:prelim:mut-op}
\end{definition}

Using the mutation operator in Definition~\ref{def:prelim:mut-op} and a fitness function $f\colon \{0,1\}^n \rightarrow \mathbb{R}$, the \EA with fixed mutation rate is defined as in Algorithm~\ref{alg_ea}.
The standard \EA uses $\chi = 1/n$ as the mutation rate for bit strings of length $n$.

\begin{algorithm}
    \caption{\EA}\label{alg_ea}
        \textbf{Input:} $0 < \chi < 1$\;
        Sample $x \in \{0,1\}^n$ uniformly at random\;
        \For{$t=0$ \KwTo $\infty$}{
            $y \assign \textsc{Mutate}(x,\chi)$\;
            \If{$f(y) \geq f(x)$}{
                $x \assign y$\;
            }
        }
\end{algorithm}

In our analysis, we use \BinVal, as given in (\ref{eq:binval}), as the fitness function.
Since it is defined such that the bit with the smallest index~$1$ has the largest weight, and the weights decrease for larger indices, we use the term ``the first $k$ bits'' in the following to refer to the most significant $k$ bits of the bit string.
We assume that the strictly increasing indices of the bits in $x$ are ordered from left to right.

For our drift analysis, we first require the multiplicative drift theorem as proven in \cite{doerr2010multiplicative}.

\begin{theorem}
    Let $(X_t)_{t\in \N}$ be an integrable process over $\CurlyBrackets{0,1}\cup S$, where $S \subset \R_{>1}$, and let $T = \inf\CurlyBrackets{t\in \N\mid X_t = 0}$.
    Assume that there is a $\delta \in \R_+$ such that, for all $s \in S \cup \CurlyBrackets{1}$ and all $t < T$, it holds that
    \begin{align*}
        \ExMath{X_t - X_{t+1} \mid X_0,\ldots,X_t} \geq \delta X_t.
    \end{align*}
    Then
    \begin{align*}
        \ExMath{T} \leq \frac{1+\ln \ExMath{X_0}}{\delta}.
    \end{align*}
    \label{thm:prelim:mult}
\end{theorem}

Additionally, we use the general version of the additive theorem, originally introduced by \cite{HE_yao_drift_2001, he_yao_drift_study_2004}, as given in \cite{kötzing_krejca_first-hitting_times_drift_2019}.

\begin{theorem}
    Let $(X_t)_{t\in \N}$ be an integrable process over $\R$, and let $T = \inf\CurlyBrackets{t\in \N\mid X_t = 0}$.
    Furthermore, suppose that the following two conditions hold (non-negativity, drift).
    \begin{itemize}
        \item[(NN)] For all $t \leq T$, $X_t \geq 0$.
        \item[(D)] There is a $\delta > 0$ such that, for all $t < T$, it holds that $\ExMath{X_t - X_{t+1} \mid X_0,\ldots,X_t} \geq \delta$.
    \end{itemize}
    Then
    \begin{align*}
        \ExMath{T} \leq \frac{\ExMath{X_0}}{\delta}.
    \end{align*}
    \label{thm:prelim:add}
\end{theorem}

%% file: sect_EFHT_EA.tex
\section{Fixed-target analysis of the \EA with fixed mutation rate}

\label{sect:anytime_basic_ea}

The following upper and lower bounds for the \EA with fixed mutation rate can be derived from \cite{WITT_2013}.
For the sake of completeness, we show it in the anytime setting to obtain simplified bounds for the optimization of the first $k \in o(n)$ bits for an arbitrary but fixed mutation rate $\chi$ for all $k \leq \frac{1}{\chi}$.

\subsection{Upper Bounds}

In order to prove an upper bound on the expected run time to optimize the first $k \in o(n)$ bits of the \EA and reuse it later for a modified version of the algorithm, we generalize first \cite[Lemma~5]{doerr2010multiplicative} in \Cref{lem:gen:mult-drift}.

\begin{lemmaE}[][normal]
    Let $n \in \mathbb{N}$ be the length of the bit string $x$ that gets optimized. 
    Consider the \EA with fixed mutation rate $0<~\chi<~1/2$ that optimizes \BinVal and let $k \in o(n)$ with $k \leq \frac{1}{\chi}$.
    For all $i~\in \{1,\ldots,k\}$, we let $w_i = 1 + \frac{k+1-i}{k}$.
    Furthermore, let $g\colon \CurlyBrackets{0,1}^n \rightarrow \R$ be the potential function with
    \begin{align*}
        g(x) &= \sum_{i=1}^{k}w_i\cdot\Brackets{1-x_i}\\
        &= \sum_{i=1}^{k}\Brackets{1 + \frac{k+1-i}{k}}\cdot\Brackets{1-x_i}
    \end{align*}
    for all $x \in \CurlyBrackets{0,1}^n$.
    Denote $x_t$ as the current best solution of iteration~$t\in\N$ and let $T = \inf\CurlyBrackets{t\in \N\mid g(x_t) \leq 0}$.
    Then, for all $t \leq T$, it holds that
    \begin{align}
        \ExMath{g(x_t) - g(x_{t+1})} \geq \frac{\chi g(x_t)}{4e}.
    \end{align}
    \label{lem:gen:mult-drift}
\end{lemmaE}

\begin{proof}[Proof sketch]
    With \BinVal as fitness function, it is possible that even for an accepted mutant with a $0$-to-$1$ flip in the first $k$ bits, the overall number of $1$ bits in this range does not increase.
    
    The intuition behind the design of the potential function $g$ is that it weights the first $k$ bits decreasing from $2$ to $1$ such that it compensates this property of \BinVal.
    We compute the expected difference of $g$ by subdividing the analysis into three cases, depending on the number of $0$-to-$1$ flips in the first $k$ bits of the generated mutant in each iteration, and considering potential $1$-to-$0$ flips that hinder the progress.
    
    Weighting the expected differences of the three cases by their probability and applying the law of total expectation conditioned on them concludes the proof and results in the lower bound given in \Cref{lem:gen:mult-drift}.
\end{proof}

\begin{proofE}
    We extend the proof of \cite[Lemma~5]{doerr2010multiplicative} such that the expected difference holds for a general $0 < \chi < 1/2$ and $k~\leq~\frac{1}{\chi}$ and therefore structure the proof of \Cref{lem:gen:mult-drift} analogously.
    In comparison to the original proof, we can neglect all bit flips besides the first $k$ indices, since \BinVal as fitness function $f$ ensures that they have no impact on the acceptance or rejection of the generated mutant as long as one of the first $k$ bits flips.
    Accordingly, only the first $k$ bits have an impact on the value of the potential function $g$.
    In the following, we prove the expected difference for every $t \leq T$.
    We denote $\Delta(x_t) = g(x_t) - g(x_{t+1})$.
    The definition of the standard \EA ensures that $f(x_{t+1}) \geq f(x_t)$ holds.
    It follows that
    \begin{align}
        \ExMath{\Delta(x_t)} = \ExMath{\Delta(x_t) \mid f(x_{t+1}) \geq f(x_t)} \PrMath{f(x_{t+1}) \geq f(x_t)}.\label{eq:gen:mult-drift:basic}
    \end{align}
    Let $x_{t,i}$ denote the bit at index $i$ of $x_t$ and let 
    \begin{align*}
        I = \CurlyBrackets{i \in \CurlyBrackets{1,\ldots,k}: x_{t,i} = 0}
    \end{align*}
    be the indices of the $0$ bits in the first $k$ bits of $x_t$.
    We use in the following the same case distinction as in the proof of \cite[Lemma~5]{doerr2010multiplicative}.
    In each case, we suppose due to (\ref{eq:gen:mult-drift:basic}) and the definition of the \EA that $f(x_{t+1}) \geq f(x_t)$ holds. With $C_1$, $C_2$, $C_3$, we denote the following events.
    \begin{itemize}
        \item[$C_1$:] There is no index $i \in I$ such that $x_{t+1,i} = 1$ (no $0$-to-$1$ flip in the first $k$ bits).
        \item[$C_2$:] There is exactly one index $i \in I$ such that $x_{t+1,i} = 1$ (one $0$-to-$1$ flip in the first $k$ bits).
        \item[$C_3$:] There are at least two different indices $j,l \in I$ such that $x_{t+1,j} = 1$ and $x_{t+1,l} = 1$ (at least two $0$-to-$1$ flips in the first $k$ bits).
    \end{itemize}
    Note that the events $C_1$, $C_2$, $C_3$ are disjoint and exhaust all possibilities. We start by considering $C_1$.
    Since event $C_1$ implies that no beneficial bit flip occurred in the first $k$ bits, we get
    \begin{align}
        \ExMath{\Delta(x_t)\mid C_1} = 0.\label{eq:gen:mult-drift:c1-final}
    \end{align}
    Next, we assume that event $C_3$ holds and show that the difference~$\Delta(x_t)$ is non-negative in expectation.
    For all $i \in \CurlyBrackets{1,\ldots,k}$, we let $g_i(x_{t,i}) = w_i \cdot (1-x_{t,i})$.
    Using the linearity of expectation, it holds that
    \begin{align}
        \ExMath{\Delta(x_t) \mid C_3} = \sum_{i= 1}^{k} \ExMath{g_i(x_{t,i})-g_i(x_{t+1,i})\mid C_3}.
        \label{eq:gen:mult-drift:c3-first-part}
    \end{align}
    Since event $C_3$ ensures that at least two indices $j,l\in I$ perform a~$0$-to-$1$ bit flip in iteration $t$ and $w_j \geq 1$ and $w_l \geq 1$, we have
    \begin{align}
        \sum_{i\in I} \ExMath{g_i(x_{t,i})-g_i(x_{t+1,i}) \mid C_3} \geq 2.
        \label{eq:gen:mult-drift:c3-second-part}
    \end{align}
    For all $1$ bits at the indices $i \in \CurlyBrackets{1,\ldots, k} \setminus I$, the expected difference in the case of $C_3$ can be bounded by the worst case scenario that they flip with probability $\chi$ to $0$.
    Then, for all $i \in \CurlyBrackets{1,\ldots, k} \setminus I$ it holds that
    \begin{align}
        \ExMath{g_i(x_{t,i})-g_i(x_{t+1,i}) \mid C_3} &\geq - w_i \PrMath{x_{t+1,i} = 0 \mid C_3\notag}\\
        &\geq -\chi w_i.
        \label{eq:gen:mult-drift:c3-third-part}
    \end{align}
    In the following, we insert (\ref{eq:gen:mult-drift:c3-second-part}) and (\ref{eq:gen:mult-drift:c3-third-part}) in (\ref{eq:gen:mult-drift:c3-first-part}).
    Using $w_i \leq 2$ for all~$i \leq k$ and $k \leq \frac{1}{\chi}$, results in
    \begin{align}
        \ExMath{\Delta(x_t) \mid C_3} &= \sum_{i\in I} \ExMath{g_i(x_{t,i})-g_i(x_{t+1,i}) \mid C_3} \notag \\
        &+ \sum_{i \in \CurlyBrackets{1,\ldots, k} \setminus I} \ExMath{g_i(x_{t,i})-g_i(x_{t+1,i}) \mid C_3}\notag\\
        &\geq 2 - \chi \sum_{i \in \CurlyBrackets{1,\ldots, k} \setminus I} w_i\notag\\
        &\geq 2 - \chi \cdot (k-2) \cdot 2\notag\\
        &\geq 0. \label{eq:gen:mult-drift:c3-final}
    \end{align}

    Finally, we consider the event $C_2$. 
    It holds with the law of total expectation and (\ref{eq:gen:mult-drift:basic}), (\ref{eq:gen:mult-drift:c1-final}), and (\ref{eq:gen:mult-drift:c3-final}), that the expected difference can be lower bounded by
    \begin{align}
        \ExMath{\Delta(x_t)} \geq \ExMath{\Delta(x_t) \mid C_2}\PrMath{C_2}.
        \label{eq:gen:mult-drift:basic-c2}
    \end{align}
    In order to bound the expected difference and probability of event~$C_2$, we denote, for every $i\in I$, $A_i$ as the event that $i$ is the only $0$ bit in the first $k$ bits of $x_t$ that flips.
    In the proof of \cite[Lemma~5]{doerr2010multiplicative}, event~$B_i$ denotes the case that $i$ is the only $0$ bit in the first $k$ bits of $x_t$ that flips, and at least a $1$ bit with index $j < i$ flips.
    Since \BinVal ensures that such a generated mutant gets rejected, we can neglect this case in our analysis.
    Therefore, the expected difference in (\ref{eq:gen:mult-drift:basic-c2}) can be bounded by
    \begin{align}
        \ExMath{\Delta(x_t)} \geq \sum_{i\in I}\ExMath{\Delta(x_t)\mid A_i}\PrMath{A_i}.
        \label{eq:gen:mult-drift:basic2-c2}
    \end{align}
    Let $i \in I$ and consider the event $A_i$.
    Then it holds that $x_{t+1,i}=1$ and, for all $j < i$, that $x_{t+1,j}=x_{t,j}$.
    In order to lower bound the expected difference in the case of $A_i$, assume that $x_{t,j} = 1$ for all $j > i$.
    Every bit at index $j > i$ flips with probability $\chi$, and since \BinVal is the fitness function, every bit flip on the right of $i$ gets accepted in the case of $A_i$.
    In the following calculation, we use $k \leq \frac{1}{\chi}$ and the finite linear sum in the second inequality and $k-i+1 \leq k$ in the third.
    After reducing the term by $\frac{i+2}{2k} > 0$ in the fourth inequality, it follows that
    \begin{align}
        \ExMath{\Delta(x_t)\mid A_i} &\geq w_i - \sum_{j=i+1}^{k}\chi w_j\notag\\
        &=1+\frac{k+1-i}{k} - \chi \sum_{j=i+1}^{k}\Brackets{1+\frac{k+1-j}{k}}\notag\\
        &=1+\frac{k+1-i}{k} - \chi (k-i) - \frac{\chi}{k}\sum_{j=1}^{k-i}j\notag\\
        &\geq 1 + \frac{k+1-i}{k}-\frac{k-i}{k} - \frac{(k-i+1)(k-i)}{2k^2}\notag\\
        &=1+\frac{1}{k}- \frac{(k-i+1)(k-i)}{2k^2}\notag\\
        &\geq 1 - \frac{k-i-2}{2k} \notag \\
        &\geq 1-\frac{k-i-2}{2k} - \frac{i+2}{2k}\notag\\
        &= \frac{1}{2}.
        \label{eq:gen:mult-drift:first-c2}
    \end{align}
    Since $\chi \leq 1/2$, the probability of $A_i$ can be lower bounded by the probability that the $i$-th bit is the only bit of the first $k$ bits that flips.
    Using $\frac{1}{k} \geq \chi$ and $\Brackets{1-(1/k)}^{k-1} \geq e^{-1}$, it holds that
    \begin{align}
        \PrMath{A_i} &\geq \chi \Brackets{1-\chi}^{k-1}\notag\\
        &\geq \chi \Brackets{1-\frac{1}{k}}^{k-1}\notag\\
        &\geq \chi e^{-1}.
        \label{eq:gen:mult-drift:second-c2}
    \end{align}
    Inserting (\ref{eq:gen:mult-drift:first-c2}) and (\ref{eq:gen:mult-drift:second-c2}) in (\ref{eq:gen:mult-drift:basic2-c2}), results in
    \begin{align}
        \ExMath{\Delta(x_t)} &\geq \frac{\chi}{e} \sum_{i\in I}\frac{1}{2}\notag\\
        &=\frac{\chi}{e} \frac{1}{4}\sum_{i\in I}2.\label{eq:gen:mult-drift:final-c2}
    \end{align}
    Using $w_i  \leq 2$ for all $i\in I$, it holds in (\ref{eq:gen:mult-drift:final-c2}) that
    \begin{align*}
        \ExMath{\Delta(x_t)} &\geq \frac{\chi}{e}\frac{1}{4}\sum_{i\in I}w_i\\
        &\geq \frac{\chi g(x_t)}{4e},
    \end{align*}
    which concludes the proof of the lemma.
\end{proofE}

Using \Cref{lem:gen:mult-drift} and applying the multiplicative drift theorem leads to the upper bounds of the \EA with fixed mutation rate~$\chi$ in the following theorem.

\begin{lemmaE}[][normal]
    Let $n \in \mathbb{N}$ be the length of the bit string $x$ that gets optimized. 
    Consider the \EA with fixed mutation rate~${0<\chi<1/2}$ that optimizes \BinVal. 
    Let $k\in o(n)$ with $k \leq \frac{1}{\chi}$.
    Then the expected number of iterations until the first $k$ bits are set to $1$ is at most~$\OMath{\frac{\log k}{\chi}}$.
    \label{thm:upper_bounds:gen_mut_rate}
\end{lemmaE}

\begin{proofE}
    In order to apply the multiplicative drift theorem, we first define the potential function $g\colon \CurlyBrackets{0,1}^n \rightarrow \R$ analogously to \Cref{lem:gen:mult-drift}. 
    For all $i \in \{1,\ldots,k\}$, we let $w_i = 1 + \frac{k+1-i}{k}$.
    Let
    \begin{align*}
        g(x) &= \sum_{i=1}^{k}w_i \cdot \Brackets{1-x_i}\\
        &= \sum_{i=1}^{k}\Brackets{1 + \frac{k+1-i}{k}} \cdot \Brackets{1-x_i}
    \end{align*}
    for all $x \in \CurlyBrackets{0,1}^n$.

    Let $T = \min\{t \mid g(x) = 0\}$.
    Using \Cref{lem:gen:mult-drift}, it holds that, for all $t \leq T$, the expected difference is 
    \begin{align*}
        \ExMath{g(x_t) - g(x_{t+1})} \geq \frac{\chi g(x_t)}{4e} = \delta g(x_t),
    \end{align*}
    with $\delta = \frac{\chi}{4e}$.
    
    Since $w_i \leq 2$ for all $i \leq k$, for $x_0$ it holds that $\ExMath{x_0}~\leq~2k$.
    Then, the expected number of iterations $T$ until the first $k$ bits of the bit string are optimized can be upper-bounded using the multiplicative drift theorem, as stated in \Cref{thm:prelim:mult}, by
    \begin{align}
        \ExMath{T} &\leq \frac{1 + \ln \ExMath{g(x_0)}}{\delta}\notag\\
        &\in \OMath{\frac{\log k}{\chi}}.
        \label{eq:upper_bounds:mult_drift}
    \end{align}
\end{proofE}

For $\chi = \frac{1}{n}$ as mutation rate of the standard \EA or respectively $\chi = \frac{c}{n}$ with $c > 0$ as fixed mutation rate of another \EA variant, the upper bounds of the expected number of iterations until the first $k$ bits are optimized follow directly from \Cref{thm:upper_bounds:gen_mut_rate} and are given in the following corollaries.

\begin{corollary}
    Let $n \in \mathbb{N}$ be the length of the bit string $x$ that gets optimized and $k \in o(n)$. 
    Consider the standard \EA optimizing \BinVal with mutation rate $\chi = \frac{1}{n}$. 
    Then, the expected number of iterations until the first $k$ bits are set to $1$ is at most $\OMath{n \log k}$.
    \label{cor:upper_bounds:standard_ea}
\end{corollary}

\begin{corollary}
    Let $n \in \mathbb{N}$ be the length of the bit string $x$ that gets optimized, $c > 0$, and $k \in o(n)$ with $k \leq \frac{n}{c}$.
    Consider the~\EA with mutation probability $\chi = \frac{c}{n}$ optimizing \BinVal. 
    Then the expected number of iterations until the first $k$ bits are set to $1$ is at most~$\OMath{\frac{n\log k}{c}}$.
    \label{cor:upper_bounds:variant_ea}
\end{corollary}

\subsection{Lower Bounds}

In order to show that our upper bounds of the expected number of iterations to optimize the first $k$ bits with the standard \EA with fixed mutation rate are tight, we bound it in the following from below.

\begin{lemmaE}[][normal]
    Let $n \in \mathbb{N}$ be the length of the bit string $x$ that gets optimized. 
    Consider the \EA with fixed mutation rate $0 < \chi < 1$ that optimizes \BinVal. 
    Let $k \in o(n)$ with $k \leq \frac{1}{\chi}$.
    Then, the expected number of iterations until the first $k$ bits are set to $1$ is at least~$\Omega\Brackets{\frac{\log k}{\chi}}$.
    \label{lem:lower_bounds:gen_mut_rate}
\end{lemmaE}

\begin{proof}[Proof sketch]
    Following the initialization uniformly at random, with probability $1/2$, at least half of the first $k$ bits are then $0$ bits, and to optimize the first $k$ bits, they must flip at least once.
    Therefore, calculating the average number of iterations until these bits have tried to flip to $1$ at least once is sufficient to get a lower bound on the expected number of iterations and concludes the proof of \Cref{lem:lower_bounds:gen_mut_rate}.
\end{proof}

\begin{proofE}
    We prove the lower bound analogously to \cite[Lemma~10]{droste2002analysis1p1ea} and generalize it for $0 < \chi < 1$ and $k \in o(n)$ with~$k \leq \frac{1}{\chi}$.
    Following the initialization uniformly at random, with probability $1/2$, at least half of the first $k$ bits are then $0$ bits.
    In order to optimize the first $k$ bits of the given bit string, every bit of those that are initialized as a $0$ bit must try to flip at least once.
    
    Let $T$ be a random variable that describes the minimum number of iterations until every $0$ bit has tried to flip at least once while generating the mutant.
    As $T$ takes only positive integers, it holds that
    \begin{align}
        \ExMath{T} &= \sum_{t=1}^{\infty} t \PrMath{T = t}\notag\\
        &= \sum_{t=1}^{\infty} \PrMath{T \geq t}.\label{eq:lower-bounds:sum}
    \end{align}

    Since it holds for all $t' \leq t$ that $\PrMath{T \geq t} \leq \PrMath{T \geq t'}$, (\ref{eq:lower-bounds:sum}) can be lower bounded for a fixed $t_0 \in \N$ by
    \begin{align}
        \ExMath{T} &= \sum_{t=1}^{\infty} \PrMath{T \geq t}\notag\\
        &\geq (t_0-1) \cdot \PrMath{T \geq t_0}.\label{eq:lower-bounds:sum:easier}
    \end{align}

    Let $t_0 = \Brackets{\frac{1}{\chi}-1}\ln (k)+1$.
    In the following, we will conclude the proof of \Cref{lem:lower_bounds:gen_mut_rate} by calculating $\PrMath{T \geq t_0}$ and inserting the result into (\ref{eq:lower-bounds:sum:easier}).
    
    Suppose that $k$ is even.
    With probability $\Brackets{1 - \chi}^{t_0-1}$ one bit does not flip in $t_0-1$ iterations.
    
    Therefore, the complementary probability $1 - \Brackets{1 - \chi}^{t_0-1}$ describes that it flips at least once in $t_0-1$ iterations.
    Extending it to~$k/2$ bits that are $0$ following the initialization uniformly at random, with probability $\Brackets{1 - \Brackets{1 - \chi}^{t_0-1}}^{k/2}$ it is the case for all of them.
    Taking again the complement $1- \Brackets{1 - \Brackets{1 - \chi}^{t_0-1}}^{k/2}$ results in the probability that at least one of the bits does not flip in the first $t_0-1$ iterations, which is equal to $\PrMath{T \geq t_0}$.
    
    Using $(1 - \chi)^{1/\chi-1} \geq e^{-1}$ in the first inequality and in the second~$(1-1/k)^{k/2} \leq e^{-1/2}$, the probability $\PrMath{T \geq t_0}$ can be lower bounded by
    \begin{align}
        \PrMath{T \geq t_0} &= 1- \Brackets{1 - \Brackets{1 - \chi}^{t_0-1}}^{k/2}\notag\\
        &= 1- \Brackets{1 - \Brackets{1 - \chi}^{(1/\chi-1) \ln k}}^{k/2}\notag\\
        &\geq 1- \Brackets{1 - e^{-\ln k}}^{k/2}\notag\\
        &=1- \Brackets{1 - \frac{1}{k}}^{k/2}\notag\\
        &\geq 1- e^{-1/2}\label{eq:lower-bounds-general-case-pr-t0}\\
        &\in \Theta\Brackets{1}.\notag
    \end{align}

    In the calculation of (\ref{eq:lower-bounds-general-case-pr-t0}), we used the fact that at least half of the $k$ bits are initialized as $0$ bits.
    Since the bit string is initialized uniformly at random, this occurs with probability of at least $\frac{1}{2}$, and we incorporate this factor in (\ref{eq:lower-bounds:sum:easier}).
    Then, setting ${t_0 = \Brackets{\frac{1}{\chi}-1}\ln(k)+1}$ and inserting~(\ref{eq:lower-bounds-general-case-pr-t0}) into (\ref{eq:lower-bounds:sum:easier}) leads to
    \begin{align*}
        \ExMath{T} &\geq \frac{1}{2} (t_0-1) \cdot \PrMath{T \geq t_0}\\
        &\geq \frac{1}{2} \Brackets{\frac{1}{\chi}-1}\ln (k) \Brackets{1- e^{-1/2}}\\
        &\in \Omega\Brackets{\frac{\log k}{\chi}}.
    \end{align*}
\end{proofE}

Given the upper bound in \Cref{thm:upper_bounds:gen_mut_rate}, \Cref{lem:lower_bounds:gen_mut_rate} matches it asymptotically, which indicates the dependency on the mutation rate $\chi$.
In the following, we introduce two corollaries of \Cref{lem:lower_bounds:gen_mut_rate} to show that the upper bounds of Corollaries~\ref{cor:upper_bounds:standard_ea} and~\ref{cor:upper_bounds:variant_ea} are also asymptotically tight.

\begin{corollary}
    Let $n \in \mathbb{N}$ be the length of the bit string $x$ that gets optimized and $k \in o(n)$. 
    Consider the standard \EA optimizing \BinVal with mutation rate $\chi = \frac{1}{n}$. 
    Then, the expected number of iterations until the first $k$ bits are set to $1$ is at least $\Omega\Brackets{n \log k}$.
    \label{cor:lower_bounds:standard_ea}
\end{corollary}

\begin{corollary}
    Let $n \in \mathbb{N}$ be the length of the bit string $x$ that gets optimized, $c > 0$, and $k \in o(n)$ with $k \leq \frac{n}{c}$.
    Consider the~\EA with mutation probability $\chi = \frac{c}{n}$ optimizing \BinVal. 
    Then, the expected number of iterations until the first $k$ bits are set to $1$ is at least~$\Omega\Brackets{\frac{n\log k}{c}}$.
    \label{cor:lower_bounds:variant_ea}
\end{corollary}

%% file: sect_EFHT_sig-cGA.tex
\section{Fixed-target analysis of the sig-cGA}
\label{sect:sig-cGA}
After analyzing the anytime performance of the standard \EA with fixed mutation rate, we investigate, in this section, the sig-cGA as introduced in \cite{doerr2018significance}.
In contrast to the (1+1) EAs, the sig-cGA is an \emph{estimation of distribution algorithm} (EDA), which stores, for each bit index $i \leq n$, a \emph{frequency} $\tau_i$.
In every iteration, two \emph{offspring} are sampled with respect to the frequency vector and compared to determine, according to the fitness function $f$, a \emph{winner} and a \emph{loser}.
In addition to the frequency vector, the algorithm stores, for each index $i \leq n$, the history vector $H_i$ of bit values of the winning offspring.
The frequencies~$\tau_i$ are set to one of three values $\CurlyBrackets{\frac{1}{n}, \frac{1}{2}, 1-\frac{1}{n}}$.
We initialize it with~$1/2$ and update the frequency~$\tau_i$ as soon as a significant number of~$1$s or~$0$s is detected in the related history~$H_i$.
For $\epsilon, \mu \in \R^+$, we define the \emph{significance} by
\begin{align}
    s(\epsilon, \mu) = \epsilon \max \CurlyBrackets{\sqrt{\mu \ln n}, \ln n}.\label{sig-cga:significance}
\end{align}
Then, the following significance function decides, for a given frequency $\tau$ and history $H$, how the frequency should be updated;
\begin{align}
    \textnormal{sig}_\epsilon(\tau_i, H_i) = \begin{cases}
        \text{up} & \text{if } \tau_i \in \CurlyBrackets{\frac{1}{n}, \frac{1}{2}} \wedge \exists m \in \N: \\
        &\DoubleVert{H_i[2^m]}_1 \geq 2^m \tau_i + s(\epsilon, 2^m\tau_i),\\
        \text{down} & \text{if } \tau_i \in \CurlyBrackets{\frac{1}{2}, 1-\frac{1}{n}} \wedge \exists m \in \N: \\
        &\DoubleVert{H_i[2^m]}_0 \geq 2^m \tau_i + s(\epsilon, 2^m\tau_i),\\
        \text{stay} & \text{else}.
    \end{cases}
    \label{sig-cga:sig-update-function}
\end{align}

The resulting algorithm is shown in \Cref{alg_sig-cga}.

\begin{algorithm}
    \caption{sig-cGA with Parameter $\epsilon$ and Significance Function $sig_\epsilon$ (\ref{sig-cga:sig-update-function}) optimizing $f$}\label{alg_sig-cga}
        \For{$i \in \CurlyBrackets{1,\ldots,n}$}{$\tau_i^{(0)} \assign \frac{1}{2}$ and $H_i \assign \emptyset$\;}
        \For{$t=1$ \KwTo $\infty$}{
            $y,z \assign$ offspring sampled with respect to $\tau^{(t)}$\;
            $x \assign$ winner of $y$ and $z$ with respect to $f$\;
            \For{$i\in \CurlyBrackets{1,\ldots,n}$}{
                $H_i \assign H_i \cup \CurlyBrackets{x_i}$\;
                \uIf{$\textnormal{sig}_\epsilon\Brackets{\tau_i^{(t)}, H_i} = \textnormal{up}$}{
                    $\tau_i^{(t+1)} \assign 1- \frac{1}{n}$\;
                }
                \uElseIf{$\textnormal{sig}_\epsilon\Brackets{\tau_i^{(t)}, H_i} = \textnormal{down}$}{
                    $\tau_i^{(t+1)} \assign \frac{1}{n}$\;
                }
                \Else{
                    $\tau_i^{(t+1)} \assign \tau_i^{(t)}$\;
                }
                \If{$\tau_i^{(t+1)} \neq \tau_i^{(t)}$}{$H_i \assign \emptyset$\;}
            }
        }
\end{algorithm}

In \cite[Theorem~2]{doerr2018significance}, Doerr and Krejca analyzed the performance of the sig-cGA on LeadingOnes. 
In order to perform the anytime analysis for the sig-cGA on \BinVal, we first introduce \Cref{lem:sig-cGA-prob-a-overline} to overcome the difference of replacing \LeadingOnes by \BinVal as the fitness function $f$.

\begin{lemmaE}[][normal]
    Let $n \in \mathbb{N}$ be the length of the bit string $x$ that gets optimized. 
    Consider the sig-cGA optimizing \BinVal with $\epsilon > 12$ being a constant.
    Let $i \in \CurlyBrackets{1,\ldots,n}$ be a fixed index and consider any of the first $\mathcal{O}(n \log n)$ iterations such that $\tau_i \in [1/n, 1/2]$. 
    Suppose that~$\tau_j = 1 - 1/n$ for all $j < i$.
    Denote $\overline{A}$ as the event that the bit at index $i$ determines the winner of the two sampled bit strings.
    Then, the probability of $\overline{A}$ can be lower bounded by $\PrMath{\overline{A}} \geq e^{-2}$.
    \label{lem:sig-cGA-prob-a-overline}
\end{lemmaE}

\begin{proof}[Proof sketch]
    With \BinVal as the fitness function of the sig-cGA in this setting, event $\overline{A}$ occurs if the two sampled offspring have the same value at each position $j < i$ and a different one at index~$i$.
    Using $\tau_j = 1 - 1/n$ for all $j < i$ to calculate the probability that both offspring have the same value at index $j$, and multiplying it for all $j < i$ results in the lower bound of $\PrMath{\overline{A}}$.
\end{proof}

\begin{proofE}
    Since \BinVal is the fitness function of the sig-cGA in this setting, event $\overline{A}$ occurs if the two sampled offspring have the same value at each position $j < i$.
    Using $\tau_j = 1 - 1/n$ for all $j < i$, both offspring have the same value at index $j$ with probability  
    \begin{align}
        \tau_j^2 + (1-\tau_j)^2 &= \Brackets{1 - \frac{1}{n}}^2 +\Brackets{\frac{1}{n}}^2\notag\\
        &= 1 - \frac{2}{n} + \frac{2}{n^2}.
        \label{eq:sig-cGA:sing_prob}
    \end{align}
    Using (\ref{eq:sig-cGA:sing_prob}) over all $j < i$, $i \leq n$ and $\Brackets{1-\frac{2}{n}}^{n-1} \geq e^{-2}$, the probability of $\overline{A}$ can be bounded by
    \begin{align*}
        \PrMath{\overline{A}} &= (\tau_j^2 + (1-\tau_j)^2)^{i-1}\\
        &= \Brackets{1 - \frac{2}{n} + \frac{2}{n^2}}^{i-1}\\
        &\geq \Brackets{1 - \frac{2}{n}}^{i-1}\\
        &\geq \Brackets{1 - \frac{2}{n}}^{n-1}\\
        &\geq e^{-2}.
    \end{align*}
\end{proofE}

By applying \Cref{lem:sig-cGA-prob-a-overline} in the proof of \cite[Theorem~2]{doerr2018significance}, \Cref{cor:sig-cGA_spec_anytime_analysis} follows.

\begin{corollaryE}[][normal]
    Let $n \in \mathbb{N}$ be the length of the bit string $x$ that gets optimized and $k \in o(n)$. 
    Consider the sig-cGA optimizing \BinVal with $\epsilon > 12$ being a constant. 
    Then, the number of iterations until the first $k$ bits are set to $1$ is in $\mathcal{O}(k\log n)$ with high probability and in expectation.
    \label{cor:sig-cGA_spec_anytime_analysis}
\end{corollaryE}

\begin{proof}[Proof sketch]
    In \cite[Theorem~2]{doerr2018significance}, Doerr and Krejca prove the global optimization time for the benchmark function \LeadingOnes, and state that the number of iterations to optimize the first~$k$ bits is in $\mathcal{O}(k\log n)$, as explained after their proof.
    To reuse their analysis, the proof must be adapted only at one step, as the other assumptions also hold for \BinVal.
    Applying \Cref{lem:sig-cGA-prob-a-overline} closes this gap and concludes the proof of \Cref{cor:sig-cGA_spec_anytime_analysis}.
\end{proof}

\begin{proofE}
    The optimization time until the sig-cGA defined in \Cref{alg_sig-cga} sets the first $k$ bits to $1$, follows directly from \cite[Theorem~2]{doerr2018significance}, as explained after their proof.
    
    In \cite[Theorem~2]{doerr2018significance}, Doerr and Krejca prove the optimization time for the benchmark function \LeadingOnes.
    In order to reuse their analysis, the proof must be adapted only at one step, as the other assumptions also hold for \BinVal.
    Consider an index $i \in \CurlyBrackets{1,\ldots,n}$ and any of the first $\mathcal{O}(n \log n)$ iterations such that $\tau_i \in [1/n, 1/2]$. 
    Suppose that $\tau_j = 1 - 1/n$ for $j < i$.
    Analogously to \cite[Theorem~2]{doerr2018significance}, let $A$ be the event that the bit at position $i$ of the winning offspring is irrelevant for the selection.
    $\overline{A}$ is the complementary event that the bit at position $i$ determines the winner of the two sampled bit strings.
    For \LeadingOnes as fitness function, event $\overline{A}$ occurs if all bits of both offspring at all indices $\CurlyBrackets{1,\ldots,i-1}$ are $1$.
    In contrast, for the \BinVal function, this is already the case if the two offspring have the same value at each position $j < i$.
    In \cite[Theorem~2]{doerr2018significance}, the lower bound $\PrMath{\overline{A}} \geq e^{-2}$ is used.
    Applying \Cref{lem:sig-cGA-prob-a-overline}, we can use the same lower bound for \BinVal.

    Since this is the only step in the proof of \cite[Theorem~2]{doerr2018significance} that differs for \LeadingOnes and \BinVal, it follows that $\tau_i$ is set to~$1~-~1/n$ in $\On((1/\tau_i)\log n)$ iterations with a probability of at least 
    \begin{align*}
        1 - n^{1 - \epsilon^2/3} \geq 1 - n^{-47}
    \end{align*}
    after all frequencies at indices $\CurlyBrackets{1,\ldots,i-1}$ are at $1-1/n$.
    The number of iterations $\On(k \log n)$ for setting the first $k$ bits to $1$ with high probability and in expectation follows.
\end{proofE}

Since the threshold $s$ as defined in (\ref{sig-cga:significance}) and the lower $(1/n)$ and upper $(1-1/n)$ bounds of the frequencies incorporate the factor $\log n$ in the upper bound of \Cref{cor:sig-cGA_spec_anytime_analysis}, fixing a $\tilde{k} < n$ and defining the threshold and frequencies depending on it, results in removing the dependency on $n$ in \Cref{cor:sig-cGA_spec_anytime_analysis-with-tilde-k} for the anytime analysis of all $k \leq \tilde{k}$.

\begin{corollaryE}[][normal]
    Let $n \in \mathbb{N}$ be the length of the bit string $x$ that gets optimized, $k, \tilde{k} \in \N$ with $k \leq \tilde{k} < n$.
    Consider the sig-cGA optimizing \BinVal with $\epsilon > 12$ being a constant.
    Replace in the sig-cGA the lower and upper bound frequencies by $1/\tilde{k}$ and $1 - 1/\tilde{k}$, and with $\mu \in \R^+$ the significance threshold $s$ by
    \begin{align*}
        s(\epsilon, \mu) = \epsilon \max \CurlyBrackets{ \sqrt{\mu \ln \tilde{k}}, \ln \tilde{k}}.
    \end{align*}
    Then, the number of iterations until the first $k$ bits are set to $1$ is in~$\mathcal{O}(k\log \tilde{k})$ with high probability and in expectation.
    \label{cor:sig-cGA_spec_anytime_analysis-with-tilde-k}
\end{corollaryE}

\begin{proofE}
    \Cref{cor:sig-cGA_spec_anytime_analysis-with-tilde-k} follows directly from the proof of \cite[Theorem~2]{doerr2018significance} as all dependencies of $n$ in the sig-cGA are replaced by $\tilde{k}$ and bit flips at indices larger than $ \tilde{k}$ do not affect whether a bit at index smaller than $\tilde{k}$ is decisive for a better or worse fitness function value due to the definition of \BinVal.
    Incorporating the adapted frequencies depending on $\tilde{k}$ into \Cref{lem:sig-cGA-prob-a-overline} and the resulting probability of $\overline{A}$ analogously to the proof of \Cref{cor:sig-cGA_spec_anytime_analysis}, ensures that the Corollary holds for \BinVal as fitness function.
\end{proofE}

In order to prove its dependency on $n$, respectively $\tilde{k}$, we show in the following the lower bound of the sig-cGA to optimize the first $k$ bits.

\begin{lemmaE}[][normal]
    Let $n \in \mathbb{N}$ be the length of the bit string $x$ that gets optimized and $k \in o(n)$. 
    Consider the sig-cGA optimizing \BinVal with $\epsilon > 12$ being a constant. 
    Then, the number of iterations until the first $k$ bits are set to $1$ is at least $\Omega\Brackets{k\log n}$ in expectation.
    \label{lem:sig-cGA_spec_anytime_analysis:lower-bound}
\end{lemmaE}

\begin{proof}[Proof sketch]
    In order to obtain the lower bound of the sig-cGA, we show that updating at least $i\in \CurlyBrackets{1,\ldots,k}$ of the first $k$ bits requires $\Omega(i \ln k)$ iterations, and sampling then the optimum has a run time of $\Omega(2^{k-i})$.
    This results in the lower bound $\Omega\Brackets{\min_{i \in \CurlyBrackets{1,\ldots,k}}\Brackets{i \log n + 2^{k-i}}}$.
    By analyzing the trade-off between the term linearly increasing in $i$ and the exponentially decreasing one, we can lower bound it by $\Omega(k \log n)$.
\end{proof}

\begin{proofE}
    Let $i \in \CurlyBrackets{1,\ldots,k}$ and denote by $T_L$ the random variable that describes the minimum number of iterations until the first $k$ bits are set to $1$ in one of the two sampled offspring.

    In the following, we consider a bit with index $j \leq k$ and frequency $\tau_j = 1/2$ 
    According to (\ref{sig-cga:sig-update-function}), for a $m \in \N$, at least $2^{m-1} + \ln n$ of the last $2^m$ entries in the history of $j$ are required to be $1$ to update $\tau_j$ to $1 - 1/n$.
    As long as the bit at index $j$ is not relevant for selecting the better offspring, $1$s and $0$s are stored with the same probability $\tau_j = 1/2$ in the history.
    Therefore, in expectation, there are at least $\Omega(\ln n)$ iterations with $j$ as the relevant bit required to update $\tau_j$.
    Thus, to update the frequency of $i$ of the first $k$ bits to $1 - 1/n$, the algorithm requires at least $\Omega(i \ln n)$ iterations.

    As soon as the frequency of $i$ bits is updated to $1 - 1/n$, the probability $p_i$ of sampling the optimum bit string where the first $k$ bits are set to $1$ is
    \begin{align}
        p_i = \Brackets{1-\frac{1}{n}}^i \Brackets{\frac{1}{2}}^{k-i}.\label{eq:sig-cga:lower-bound:p}
    \end{align}
    Since two offspring are sampled in every iteration, the random variable $T_L$ is geometrically distributed with parameter $2p$.
    Using~(\ref{eq:sig-cga:lower-bound:p}), $\ExMath{T_{L,i}}$ when $i$ bits are updated can be lower bounded by
    \begin{align*}
        \ExMath{T_{L,i}} &= \frac{1}{2p_i}\\
        &= \frac{1}{2\Brackets{1- \frac{1}{n}}^{i}\Brackets{\frac{1}{2}}^{k-i}}\\
        &= \Omega\Brackets{2^{k-i}}.
    \end{align*}

    Since the optimum can be sampled for each $i \in \CurlyBrackets{1,\ldots,k}$, the run time can be lower-bounded by
    \begin{align}
        \ExMath{T_L} = \Omega\Brackets{\min_{i \in \CurlyBrackets{1,\ldots,k}}\Brackets{i \log n + 2^{k-i}}}.\label{eq:sig-cga:lower-bound:lower-bound-min}
    \end{align}

    In order to bound the term in Equation~\eqref{eq:sig-cga:lower-bound:lower-bound-min}, we observe a trade-off between a term increasing linearly in $i$ and one decreasing exponentially.
    The minimum is attained when the two terms are balanced, specifically at $i \approx k - \log \log n$.
    Substituting this value, the expression simplifies to $\Omega(k\log n- \log \log n)$.
    Given that $k<n$, this is asymptotically dominated by $\Omega(k\log n)$, which concludes the proof of the lemma. 
\end{proofE}

Replacing again $n$ by $\tilde{k}$ in the frequencies and the significance threshold, the lower bound for this scenario follows directly from\Cref{lem:sig-cGA_spec_anytime_analysis:lower-bound}.

\begin{corollary}
    Let $n \in \mathbb{N}$ be the length of the bit string $x$ that gets optimized, $k, \tilde{k} \in \N$ with $k \leq \tilde{k} < n$.
    Consider the sig-cGA optimizing \BinVal with $\epsilon > 12$ being a constant.
    Replace in the sig-cGA the lower and upper bound frequencies by $1/\tilde{k}$ and $1 - 1/\tilde{k}$, and with $\mu \in \R^+$ the significance threshold $s$ by
    \begin{align*}
        s(\epsilon, \mu) = \epsilon \max \CurlyBrackets{ \sqrt{\mu \ln \tilde{k}}, \ln \tilde{k}}.
    \end{align*}
    Then, the number of iterations until the first $k$ bits are set to $1$ is at least $\Omega\Brackets{k\log \tilde{k}}$ in expectation.
    \label{cor:sig-cGA_spec_anytime_analysis-with-tilde-k:lower-bound}
\end{corollary}

%% file: sect_EFHT_EA_AdjustingMR.tex
\section{Fixed-target analysis of the \EA with adjusting mutation rate}
\label{sect:ea-adjusting}
We consider now an idealized \EA with an adjusting mutation rate during its optimization, in order to show later that the run time of a realistic self-adjusting variant is not far off this idealized one.

The intuition behind the algorithm is that we divide the bit string $x$ that needs to be optimized into bit blocks of a specific size, containing the first bit indices.
For every $m \in \N$ with $m \leq \log_2 n$, we define the block that contains the first $2^m$ bits of $x$.
According to~\cite[Theorem~3.1]{WITT_2013}, a fixed mutation rate of $1/n$ of the \EA is optimal to optimize a bit string with length $n$.
Since we use \BinVal as the fitness function, only the bit flips of the first $2^m$ bits have an impact on their optimization, and a mutation rate of $1/2^m$ is optimal for this purpose.
Therefore, we define the \EA with adjusting mutation rate as shown in Algorithm~\ref{alg_ea_adjusted_mr} such that it chooses iteratively the greedy optimal mutation rate for the smallest possible block in which the first bit, that is not optimized yet, is located.
For $i \in \CurlyBrackets{1,\ldots,n}$ as the smallest bit index with~$x_i = 0$, it adjusts the mutation rate~$\chi$ to~$1/2^m$ with~$m=\lceil \log_2 i \rceil$.
Besides that, the algorithm optimizes the bit string in the same way as the standard \EA with fixed mutation rate.

\begin{algorithm}
    \caption{\EA with adjusting MR}\label{alg_ea_adjusted_mr}
        Sample $x \in \{0,1\}^n$ uniformly at random\;
        \For{$t=0$ \KwTo $\infty$}{
            $\chi \assign \textsc{OptimalMutationRate}(x)$\;
            $y \assign \textsc{Mutate}(x,\chi)$\;
            \If{$f(y) \geq f(x)$}{
                $x \assign y$\;
            }
        }
\end{algorithm}

As shown in \Cref{alg_ea_adjusted_mr}, we assume that we have an oracle function $\textsc{OptimalMutationRate}$ that gives us the greedy optimal mutation rate for the bit string $x$ in every iteration before applying the mutation operator, as given in \Cref{def:prelim:mut-op}.
For our scenario of optimizing the \BinVal fitness function, we define the following oracle function as described in \Cref{alg_opt_mr_function}.

\begin{algorithm}
    \caption{The Greedy Optimal MR function for \BinVal}\label{alg_opt_mr_function}
        Let $x \in \{0,1\}^n$\;
        \For{$i \in \CurlyBrackets{1,\ldots,n}$}{
            \If{$x[i] = 0$}{
                $m \assign \lceil \log_2 i \rceil$\;
                \Return $\chi = 1/2^m$\;
            }
        }
        \Return $\chi = 1/n$\;
\end{algorithm}

The \EA does not have to set the bits strictly from the first to the last one while optimizing \BinVal.
In order to use this characteristic for the run time analysis of the \EA with adjusting mutation rate and to align it with the definition of the bit index blocks, we analyze the optimization of \BinVal in phases.
Phase~$r$ with $1 \leq r \leq \log_2 n$ starts as soon as the first $2^{r-1}$ bits are set to $1$.
According to \cite[Theorem~3.1]{WITT_2013} and \Cref{alg_opt_mr_function}, the greedy optimal mutation rate remains at $1/2^{r}$ during Phase $r$.
Using the definition of \BinVal, as soon as the algorithm progresses to Phase~$r$, it is not possible to go back to a previous phase since the already optimized first~$2^{r-1}$ bits do not flip back to $0$.

\begin{lemmaE}[][normal]
    Let $n \in \N$ be the length of the bit string that gets optimized and $r \in \N$ with $2^{r} \leq n$.
    Consider the \EA with adjusting mutation rate and the corresponding oracle function for the greedy optimal mutation rate that optimizes \BinVal.
    We define the phases of the algorithm as follows.
    In Phase~$r$, the first $2^{r-1}$ bits of the bit string are already set to $1$, and we transition to the next Phase~$r+1$ as soon as the first $2^{r}$ bits are set.
    Suppose we are in Phase~$r$ of the algorithm.
    Then the expected number of iterations until the first $2^{r}$ bits are set to $1$ is in $\On(2^r \log 2^r)$.
\label{lem_runtime_per_phase}
\end{lemmaE}

\begin{proofE}
    Let $r \in \N$ with $2^{r} \leq n$ be the current phase.
    To apply the multiplicative drift theorem, we first define the potential function $g\colon \CurlyBrackets{0,1}^{2^r} \rightarrow \R$ analogously to \Cref{lem:gen:mult-drift}. 
    For all $i \in \{1,\ldots,k\}$, we denote $w_i = 1 + \frac{2^r+1-i}{2^r}$.
    Let
    \begin{align*}
        g(x) &= \sum_{i=1}^{2^r}w_i\cdot\Brackets{1-x_i}\\
        &= \sum_{i=1}^{2^r}\Brackets{1 + \frac{2^r+1-i}{2^r}}\cdot\Brackets{1-x_i}
    \end{align*}
    for all $x \in \CurlyBrackets{0,1}^{2^r}$.

    Let $T = \min\{t \mid g(x) = 0\}$.
    Using \Cref{lem:gen:mult-drift}, it holds that for all $t \leq T$ and $\chi = 1/2^r$, the expected difference is 
    \begin{align*}
        \ExMath{g(x_t) - g(x_{t+1})} &\geq \frac{g(x_t)}{4e2^r} = \delta g(x_t),
    \end{align*}
    with $\delta = \frac{1}{4e2^r}$.
    Since $w_i \leq 2$ for all $i \leq 2^r$ and the phase definition implies that at least the first $2^{r-1}$ bits are already set to $1$ at the beginning of Phase~$r$, for $x_0$ it holds that $\ExMath{g(x_0)} \leq 2^r$.
    
    The expected number of iterations $T$ until the first $2^r$ bits of the bit string are optimized can be upper-bounded using the multiplicative drift theorem, as stated in \Cref{thm:prelim:mult}, by
    \begin{align*}
        \ExMath{T} &\leq \frac{1 + \ln \ExMath{g(x_0)}}{\delta}\\
        &= 4e2^r \Brackets{1+\ln 2^r}\\
        &\in \On\Brackets{2^r\log 2^r}.
    \end{align*}
\end{proofE}

In the following theorem, we use \Cref{lem_runtime_per_phase} and our phase definition to bound the expected number of iterations until the first~$k$ bits are optimized.

\begin{theoremE}[][normal]
    Let $n \in \mathbb{N}$ be the length of the bit string $x$ that gets optimized and $k \in o(n)$. 
    Consider the \EA with adjusting mutation rate $\chi$ and the oracle function as described in \Cref{alg_opt_mr_function} with \BinVal as fitness function. 
    Then the expected number of iterations until the first $k$ bits are set to $1$ is $\On(k\log k)$ in expectation.
    \label{thm:ea-with-adj-mut-rate-run-time}
\end{theoremE}

\begin{proof}[Proof sketch]
    The intuition behind the proof of \Cref{thm:ea-with-adj-mut-rate-run-time} is that, since the number of optimized leading bits is doubled in every iteration, it takes $\lceil \log k \rceil$ phases until the first $k$ bits are optimized.
    Summing the iterations of all phases until then and using the dominance of the largest sum concludes the proof.
\end{proof}

\begin{proofE}
    By applying \Cref{lem_runtime_per_phase}, it holds that the \EA with adjusting mutation rate takes $\On(2^r \log 2^r)$ iterations in expectation to double the number of leading $1$ bits from~$2^{r-1}$ to $2^{r}$.
    Since the \BinVal function ensures that the optimized leading bits cannot be set to $0$ again, the algorithm does not go back to Phase $r-1$ when the first $2^{r-1}$ bits are set to $1$.
    As we are doubling the number of leading bits in each phase, it takes $\lceil \log k \rceil$ phases until the first $k$ bits are optimized.
    Summing up the iterations until the first $k$ bits are set to $1$ and using the dominance of the largest sum term, results in an expected number of iterations of at most
    \begin{align*}
        \sum_{r = 0}^{\log k} 2^r \log 2^r \in \On(k \log k).
    \end{align*}
\end{proofE}

In order to show the optimality of the \EA with adjusting mutation rate to optimize the first $k$ bits for all $k$ simultaneously, we derive the following corollary from \Cref{lem:lower_bounds:gen_mut_rate} by using the stochastic dominance of the \EA with fixed mutation rate for a known $k$.

\begin{corollaryE}[][restate,end,category=eaAdjMr,text link=]
    Let $n \in \mathbb{N}$ be the length of the bit string $x$ that gets optimized and $k \in o(n)$. 
    Consider the \EA with adjusting mutation rate $\chi$ and the oracle function as described in \Cref{alg_opt_mr_function} with \BinVal as fitness function. 
    Then the expected number of iterations until the first $k$ bits are set to $1$ is at least $\Omega(k\log k)$ in expectation.
    \label{cor:ea-with-adj-mut-rate-run-time-low-bound}
\end{corollaryE}

\begin{proofE}
    Assume we know the number of the first $k \in o(n)$ bits that should be optimized.
    Then, $\chi=1/k$ is the optimal fixed mutation rate for the standard \EA to optimize the first $k$ bits of a bit string with \BinVal as fitness function according to \cite[Theorem~3.1]{WITT_2013}.
    Until the first $k$ bits are optimized, the \EA with adjusting mutation rate has a mutation rate of $\chi \geq 2^m$ with $m = \lceil \log_2 k \rceil$.
    Thus, for a known $k$, its optimization of the first~$k$ bits gets stochastically dominated by the \EA with a fixed mutation rate of $\chi=1/k$.
    Therefore, the lower bound of \Cref{lem:lower_bounds:gen_mut_rate} can be applied and $\Omega(k\log k)$ follows.
\end{proofE}

It follows that the upper bound in \Cref{thm:ea-with-adj-mut-rate-run-time} is asymptotically tight.
Since it is, up to leading constants, equal to the expected upper bound of the fixed-target analysis for a known $k \in o(n)$ and the~\EA with the best fixed optimal mutation rate $\chi = 1/k$, this shows the strength of the \EA with adjusting mutation rate as an anytime algorithm.

%% file: sect_EA_Self-Adjusting-MR.tex
\section{Fixed-target analysis of the \EA with self-adjusting mutation rate}
\label{sect:ea-self-adj}
Usually, we do not have an approximate oracle function, as in Algorithm~\ref{alg_ea_adjusted_mr}, that returns the greedy optimal mutation rate in every iteration to increase the chances of making progress in the next bit block.
Therefore, we introduce, analogously to~\cite{doerr2019self}, a variant of the \EA that self-adjusts the mutation rate $\chi_t$ multiplicatively during its optimization of the bit string $x$.
The self-adjustment is based on the acceptance or rejection of the generated mutant.
In the following subsection, we introduce our variant of the \EA with self-adjusting mutation rate and analyze its anytime performance theoretically.

\subsection{Algorithm description}

As shown in \Cref{alg_ea_self_adjust_mr}, the algorithm does not differ from the standard \EA except for the self-adjustment of the mutation rate instead of keeping a fixed.
The mutation operator, as given in \Cref{def:prelim:mut-op}, is applied with this self-adjusting rate.
Let~$\chi_{\max}~=~1/2$ and~$\chi_{\min}~=n^{-m}$ for a constant~$m>1$.
Starting with an initial mutation rate $\chi_0 \in [\chi_{\min}, \chi_{\max}]$, we use two parameters $1 < a < 2$ and $0< b < 1$ to increase or decrease the mutation rate $\chi_t$ in the current iteration~$t$.
If the created mutant~$x_t'$ has at least the same fitness value as the former best known bit string~$x$, we increase the mutation rate~$\chi_t$ by multiplying it by~$a$.
Otherwise, we decrease it by multiplying it by~$b$.
The mutation rate~$\chi_t$ is bounded by the upper~$\chi_{\max}$ and lower limit~$\chi_{\min}$. 

\begin{algorithm}[htp]
    \caption{\EA with self-adjusting MR}\label{alg_ea_self_adjust_mr}
    \textbf{Parameter:} $\chi_0 \in [\chi_{\min}, \chi_{\max}]$\;
    Sample $x_0 \in \{0,1\}^n$ uniformly at random\;
    \For{$t=0$ \KwTo $\infty$}{
        $x_t' \assign \textsc{Mutate}(x_t,\chi_t)$\;
        \uIf{$f(x_t') \geq f(x_t)$}{
            $x_{t+1} \assign x_t'$\;
            $\chi_{t+1} \assign \min\{a\chi_t, \chi_{\max}\}$\;
        }
        \Else
        {
            $x_{t+1} \assign x_t$\;
            $\chi_{t+1} \assign \max\{b\chi_t, \chi_{\min}\}$\;
        }
    }
\end{algorithm}

The intuition behind the algorithm is that it should approach the optimal mutation rate to make progress in the next bit block that is not yet optimized by self-adjusting its mutation rate.
If a large amount of the first $2^r$ bits with $2^r < n/2$ are already set to~$1$, a mutation rate of $\chi = 1/2^{r+1}$ is optimal to optimize the next~$2^r$ bits according to \cite[Theorem~3.1]{WITT_2013}.
The self-adaptation of the mutation rate works as follows.
For a mutation rate larger than the optimal, the probability of flipping the already set bits of the first $2^r$ bits as the leftmost bit increases, resulting in the rejection of the generated mutant and, consequently, a decrease in the mutation rate.
A mutation rate that is significantly smaller than the optimal range reduces the probability of flipping the already optimized bits in the first $2^r$ bits but also the chances of making progress in the next~$2^r$ bits.
A smaller mutation rate also leads to optimizations of lower-order bits, which increases the mutation rate and brings it closer to the optimal range for the right choice of the parameters~$a$~and~$b$.

\subsection{Fixed-target analysis}

In order to analyze the fixed-target run time of the presented \EA with self-adjusting mutation rate and incorporate the characteristics of the \BinVal function, we divide its optimization progress into phases.
We define the phases and their transitions as follows.
Let~$r \in \N$ with $2^r \leq n/2$.
In Phase~$r$, we analyze the optimization of the first $2^r$ bits and therefore count the number of $0$ bits in this area.
We fix a constant $0<\gamma < 1-\ln(2)$ and perform the transition to the next Phase~$r+1$ as soon as the next decrease of the number of~$0$ bits in the first $2^r$ bits would result in fewer than $\gamma 2^r$ bits of value $0$ in this range.
The intuition behind this phase definition is that we can use it to determine upper and lower bounds for the distribution of the $0$ and $1$ bits at the indices larger than $2^r$, which is required for our drift analysis. 

In the following, all $\On$-Notations depend on the bit length $n \in \N$ and the number $k \in o(n)$ of bits that should be optimized.
The parameter $\gamma$ is a fixed coefficient.

Before we perform the anytime analysis of the \EA with self-adjusting mutation rate on the \BinVal fitness function, we will first analyze how many iterations the algorithm requires to transition to the next phase.

\begin{lemmaE}[][normal]
    \label{lem_self_adj_mut_phase_analysis}
    Let $n \in \N$ be the length of the bit string that gets optimized, $r \in \N$ with $2^r \leq n/2$, and $\gamma$ a constant with~$0~<~\gamma~<~1~-~\ln(2)$.
    Consider the \EA with self-adjusting mutation rate that optimizes \BinVal.
    Let $1 < a < 2$ and $0 < b < 1$ be the parameters of the algorithm such that $\log_2 a/(-\log_2 b)~>~(1-\gamma)/\gamma$ holds (one can choose, for example, $\gamma~=~0.15$, $a~=~1.85$ and $b~=~0.898$).
    Assume the algorithm is in Phase~$r$ and tries to optimize the first $2^r$~bits.
    It transitions to Phase~$r+1$ as soon as the next decrease of the number of $0$ bits in the first $2^r$ bits would result in fewer than $\gamma 2^r$ bits of value $0$ in this range.
    The expected number of iterations until the algorithm reaches Phase~$r+1$ is in~$\On(2^r \log 2^r)$.
\end{lemmaE}

\begin{proof}[Proof sketch]
    In order to prove \Cref{lem_self_adj_mut_phase_analysis}, we introduce a new combined potential function that handles the optimization of the first $2^r$ bits and the adjustment of the mutation rate towards its optimal value $\chi_{\text{opt}} = 1/2^r$ in Phase~$r$ simultaneously.
    Let $\alpha > 2$ and~$d>1$ be constants, $c \in \Theta(1/2^r)$, and $t\in \N$.
    For the optimization of the first $2^r$ bits, we define analogously to \cite[Theorem~4.1]{WITT_2013} the function 
    \begin{align*}
        g(x) &= \sum_{i=1}^{2^r} \Brackets{1+\frac{\alpha d \chi_{\text{opt}}}{\Brackets{1-d\chi_{\text{opt}}}^{2^r-1}}}^{2^r-i} (1-x_{i}).
    \end{align*}
    Following the idea of \cite[Theorem~3]{doerr2010multiplicative}, we scale $g$ logarithmically with $g_{\min} \geq 1$ and combine it additively with a second term that weights the adjustment of the mutation rate towards its optimal value.
    The resulting combined potential function is
    \begin{align*}
        h(x,\chi) &= \ln \Brackets{1+ \frac{g(x)}{g_{\min}}} + c \left \lvert \log_2 \frac{\chi}{\chi_{\text{opt}}}\right\rvert.
    \end{align*}
    Let $T$ denote the first iteration in which the algorithm leaves Phase~$r$.
    Using $\ExMath{h(x_0,\chi_0)}\in \OMath{\log 2^r}$ for the initial values $x_0$ and $\chi_0$, we conclude the proof of \Cref{lem_self_adj_mut_phase_analysis} by showing
    \begin{align}
        \ExMath{h(x_t, \chi_t) - h(x_{t+1}, \chi_{t+1}) \mid x_t, \chi_t} \geq \delta \in \Theta(1/2^r)
        \label{eq_drift_condition_cases_sketch}
    \end{align}
    for all $t \leq T$ and applying the additive drift theorem, as defined in \Cref{thm:prelim:add}.
    
    We subdivide the drift analysis into three cases according to different ranges of the mutation rate $\chi_t$ and show for all of them that (\ref{eq_drift_condition_cases_sketch}) holds.
    The ranges of $\chi_t$ corresponding to the three cases are $[\chi_{\min}, \chi_{\text{opt}}/d]$ (Case~I), $[\chi_{\text{opt}}/d, d\chi_{\text{opt}}]$ (Case~II) and $[d\chi_{\text{opt}}, \chi_{\max}]$ (Case~III).
    
    For $\chi_t \in [\chi_{\min}, \chi_{\text{opt}}/d]$, the intuition of Case~I is that the mutation rate is too small to make significant progress in the first $2^r$ bits.
    Using our phase definition, it is possible to analyze the distribution of the less signi\-ficant bits and show that the algorithm accepts enough mutations introduced by flipping these bits, allowing it to self-adjust the mutation rate towards its optimal value.
    The required drift stems then from the second term of the potential function $h$.

    For Case~II, we obtain the required drift due to the optimizations in the first $2^r$ bits, since the mutation rate $\chi_t\in [\chi_{\text{opt}}/d, d\chi_{\text{opt}}]$ is in the optimal range around $\chi_{\text{opt}}$.
    We show this by analyzing the potential function $g$ and proving the equivalence between the multiplicative and additive drift for our definition of the first term of $h$.
    The parameter $c$ ensures that potential adjustments of the mutation rate do not eliminate the expected drift.

    We consider the mutation rates $\chi_t \in [d\chi_{\text{opt}}, \chi_{\max}]$ of Case~III as too large to make significant progress in the first $2^r$ bits, as it becomes too likely that already optimized bits flip as the leftmost bit, resulting in a rejection of the generated mutant.
    We show that the implied self-adjustments of the mutation rate towards its optimal range are sufficient to obtain the required drift in the second term of $h$.

    As \BinVal is our fitness function, it is possible that even for a generated mutant with a leftmost $0$-to-$1$ bit flip in the first $2^r$ bits, the total number of $1$ bits in this block does not increase or even decrease.
    The exponentially increasing weights of the potential function $g$ ensure that we can bound the expected difference for the first term of $h$ such that it does not eliminate the required drift of the second term in the Cases~I and~III.
\end{proof}

\begin{proofE}
    Let $d > 1$ be a constant such that 
    \begin{align}
        \frac{1-\gamma}{\gamma} + \frac{1}{(d-1)\gamma} < \frac{\log_2 a}{- \log_2 b} < 2d\left(\frac{3}{4} - \frac{\gamma}{2}\right)+1 \label{eq_constraint_d}
    \end{align}
    and $d \in \On(1)$  holds.
    For fixed parameters $a$, $b$, and $\gamma$, the terms on the left of the first inequality of (\ref{eq_constraint_d}) decrease monotonically for larger $d$ and converge to $\frac{1-\gamma}{\gamma}$.
    The terms on the right of the second inequality of (\ref{eq_constraint_d}) increase monotonically unbounded for larger $d$.
    Therefore, for any parameter setting of $a$, $b$ and $\gamma$ as defined, there exists a
    \begin{align*}
        d > \max \left \{\frac{-\log_2 b}{\gamma \log_2 a + (1-\gamma) \log_2 b} + 1, \frac{1}{3-2\gamma} \frac{\log_2 a}{- \log_2 b} - 1\right\}
    \end{align*}
    such that (\ref{eq_constraint_d}) is fulfilled.
    Using that $d$ is only restricted by its lower bound, we assume $d > a$ and $d > 1/b$.
    Let
    \begin{align*}
        c < \min\left\{\frac{2d-1}{d^22^r}, \frac{d}{e^d2^r}\right\} \cdot \min\CurlyBrackets{\frac{1}{\log_2 a}, \frac{1}{-\log_2 b}} \cdot \frac{1}{4} \in \Theta(1/2^r).
    \end{align*}
    For our example parameter setting with $\gamma=0.15$, $a=1.85$ and $b=0.898$, the constraints are fulfilled for any $d \geq 13$ and 
    \begin{align*}
        c<\frac{13}{e^{13}2^r}\cdot \frac{1}{-\log_2 0.898}\cdot \frac{1}{4}.
    \end{align*}
    
    Due to the phase definition and characteristics of the \BinVal fitness function, we need to make more progress in the lower ordered~$2^{r-1}$ of the first $2^r$ bits to transition to the next phase.
    To achieve this, the greedy optimal mutation rate is $\chi_{\text{opt}} = 1/2^r$ according to \cite[Theorem~3.1]{WITT_2013} and the properties of \BinVal.

    We denote the bit at index $i \leq n$ of $x_t$ by $x_{t,i}$.
    In order to define a combined potential function, we use for the optimization of the first $2^r$ bits a similar approach to Witt in \cite[Theorem~4.1]{WITT_2013}.
    For an~$\alpha \in \Theta(1)$ with~$\alpha > 2$ and 
    \begin{align*}
        g_i=\Brackets{1+\frac{\alpha d \chi_{\text{opt}}}{\Brackets{1-d\chi_{\text{opt}}}^{2^r-1}}}^{2^r-i},
    \end{align*}
    we define
    \begin{align*}
        g(x_t) &= \sum_{i=1}^{2^r} g_i \cdot (1-x_{t,i})\\
        &= \sum_{i=1}^{2^r} \Brackets{1+\frac{\alpha d \chi_{\text{opt}}}{\Brackets{1-d\chi_{\text{opt}}}^{2^r-1}}}^{2^r-i} (1-x_{t,i}).
    \end{align*}
    Following the idea of \cite[Theorem~3]{doerr2010multiplicative}, we scale the function $g(x_t)$ logarithmically to combine it additively with the other part of the combined potential function.
    Let $g_{\min} \geq 1$ be the minimum function value of $g(x_t)$ for every current best solution that contains at least the minimum number of $\gamma 2^r$ bits of value $0$ in the first $2^r$ bits.
    Then, 
    \begin{align}
        h_1(x_t) &= \ln \Brackets{1+ \frac{g(x_t)}{g_{\min}}}. \label{eq:big-drift:def-h1}
    \end{align}
    For the mutation rate $\chi_t$, we define a potential function that handles its adjustment towards the optimal range to increase the probability of making progress.
    For this purpose, let
    \begin{align*}
        h_2(\chi_t) = c \left \lvert \log_2 \frac{\chi_t}{\chi_{\text{opt}}}\right\rvert.
    \end{align*}
    As we want to show additive drift for this process, we define the combined potential function as the sum of both parts, which results in
    \begin{align*}
        h(x_t,\chi_t) &= h_1(x_t) + h_2(\chi_t)\\
        &= \ln \Brackets{1+ \frac{g(x_t)}{g_{\min}}} + c \left \lvert \log_2 \frac{\chi_t}{\chi_{\text{opt}}}\right\rvert.
    \end{align*}

    The expected initial value of the potential function depends on the initial bit string $x_0$ and the mutation rate $\chi_0\leq\chi_{\max}=1/2$ at the beginning of Phase~$r$.

    Using the finite linear sum in the second inequality and the $e$-function in the third, the initial value of $g(x_0)$ can be upper-bounded by
    \begin{align}
        g(x_0) &\leq \sum_{i=1}^{2^r} g_i\notag\\
        &\leq \frac{\Brackets{1+ \frac{\alpha d\chi_{\text{opt}}}{\Brackets{1-d\chi_{\text{opt}}}^{2^r-1}}}^{2^r} - 1}{\alpha d\chi_{\text{opt}}\Brackets{1-d\chi_{\text{opt}}}^{1-2^r}}\notag\\
        &\leq \frac{e^{2^r\alpha d \chi_{\text{opt}}\Brackets{1-d\chi_{\text{opt}}}^{1-2^r}}}{\alpha d\chi_{\text{opt}}\Brackets{1-d\chi_{\text{opt}}}^{1-2^r}}.\label{eq:big-drift:g-x0}
    \end{align}

    Bounding $g_{\min} \geq 1$, and inserting (\ref{eq:big-drift:g-x0}) in (\ref{eq:big-drift:def-h1}) by taking the logarithm leads to
    \begin{align*}
        h_1(x_0) &\leq 2^r \alpha d \chi_{\text{opt}} \Brackets{1-d\chi_{\text{opt}}}^{1-2^r} \\
        &+ \ln \Brackets{1/(d \chi_{\text{opt}})} + \ln \Brackets{\Brackets{1-d\chi_{\text{opt}}}^{2^r-1}}.
    \end{align*}
    Applying $\alpha,d \in \Theta(1)$, $\chi_{\text{opt}} = 1/2^r$, $\Brackets{1-d\chi_{\text{opt}}}^{1-2^r} \in \Theta(1)$ and  $\Brackets{1-d\chi_{\text{opt}}}^{2^r-1} \in \Theta(1)$, results in
    \begin{align}
        h_1(x_0) \in \On\Brackets{\log 2^r}.
        \label{eq:drift:upper-bound-initial-h1}
    \end{align}
    
    Using (\ref{eq:drift:upper-bound-initial-h1}), $\chi_{0} \leq 1/2$ and $c \in \Theta(1/2^r)$, the combined expected initial value can be upper-bounded by
    \begin{align*}
        \ExMath{h(x_0,\chi_{0})} &= h_1(x_0) + c \left\lvert\log_2 \frac{\chi_{0}}{\chi_{\text{opt}}}\right\rvert\\
        &\leq h_1(x_0) + c \left\lvert\log_2  \frac{2^r}{2}\right\rvert\\
        &\in \On(\log 2^r).
    \end{align*}

    Let $\DoubleVert{x_t}_0$ denote the number of zero bits in the first $2^r$ bits of $x_t$ in iteration $t \in \N$.
    Let $T = \inf\CurlyBrackets{t\in\N\mid \DoubleVert{x_t}_0\leq\gamma2^r}$ and $T' = \inf\{t\in\N\mid h(x_t,\chi_t)\leq \ln2\}$.
    By definition, it holds that $T~<~T'$.
    To conclude the proof of \Cref{lem_self_adj_mut_phase_analysis}, we have to show that the drift condition 
    \begin{align}
        \ExMath{h(x_t, \chi_t) - h(x_{t+1}, \chi_{t+1}) \mid x_t, \chi_t} \geq \delta 
        \label{eq_drift_condition_cases}
    \end{align}
    with $\delta \in \Theta(1/2^r)$ holds for all $t < T$.
    Then, according to the additive drift theorem in \Cref{thm:prelim:add}, it follows that
    \begin{align*}
        \ExMath{T} &\leq \frac{\ExMath{h(x_0, \chi_0)}}{\delta} \in \OMath{2^r\log 2^r}.
    \end{align*}

    In each iteration of the algorithm, we distinguish the possible effects on the potential function by two different events of the leftmost bit flip of the generated mutant.
    Let event $A$ denote that at least one bit flip occurs and the leftmost bit flip is in the first~$2^r$ bits, and let $B$ denote that the leftmost bit flip is a $0$-to-$1$ flip or that there are no bit flips.
    The negation $\overline{A}$ denotes accordingly that the leftmost bit flip is not in the first $2^r$ bits or no bit flips, and $\overline{B}$ that a~$1$-to-$0$ flip occurs as the leftmost flip.
    Note that the event $B$ controls if the mutation rate $\chi_t$ changes by multiplying it by $a$ in the case of $B$ or otherwise by $b$ in this iteration, since we use \BinVal as the fitness function of the algorithm.
    We include the possibility that no bit flips in event $B$ since it has the same impact on $\chi_t$.
    Using the intersection of $A$ and $B$, we can define four different events in each iteration.
    
    We are interested in all four combinations.
    \begin{enumerate}
        \item[] $A \cap B$: Leftmost bit flip is a $0$-to-$1$ flip in the first $2^r$ bits
        \item[] $A \cap \overline{B}$: Leftmost bit flip is a $1$-to-$0$ flip in the first $2^r$ bits
        \item[] $\overline{A} \cap B$: Leftmost bit flip is a $0$-to-$1$ flip after the first $2^r$ bits or no bit flips
        \item[] $\overline{A} \cap \overline{B}$: Leftmost bit flip is a $1$-to-$0$ flip after the first $2^r$ bits
    \end{enumerate}
    The number of $0$ bits in the first $2^r$ bits, denoted by $\DoubleVert{x_t}_0$, changes only in the intersection of $A$ and $B$.
    Then, the function value of $h_1$ changes by the weighted difference of the flipped $0$ and $1$ bits in the first $2^r$ bits.

    In the following, we denote the expected difference by
    \begin{align*}
        \ExMath{\Delta_t h\mid x_t,\chi_t} = \ExMath{h(x_t, \chi_t) - h(x_{t+1}, \chi_{t+1}) \mid x_t, \chi_t},
    \end{align*}
    and subdivide the drift analysis into three cases according to different ranges of the mutation rate $\chi_t$.
    The ranges of $\chi_t$ corresponding to the three cases are $[\chi_{\min}, \chi_{\text{opt}}/d]$~(Case~I), $[\chi_{\text{opt}}/d, d\chi_{\text{opt}}]$ (Case~II) and $[d\chi_{\text{opt}}, \chi_{\max}]$~(Case~III).
    
    The intuition behind the case distinction is that the algorithm makes enough progress with a mutation rate in the range of Case~II to obtain the required drift.
    Since in the other cases, Cases~I and~III, the mutation rate differs too much from its optimal value, the probability of making progress is too small to get the required drift. 
    We show instead that it stems from the adaptation of the mutation rate.
    In order to do this, we use Claim~\ref{claim:big-drift:lower-bound-drift}, which states that, in expectation, the drift of $h_1(x_t)$ is positive in the Cases~I~and~II and the absolute value of the possible negative difference in Case~III is in the order of $o\Brackets{1/2^r}$.
    We introduce Claim~\ref{claim:big-drift:change-g2} to bound the change of $h_2(\chi_t)$ depending on the acceptance or rejection of the generated mutant.

    \begin{claim}
        Let $\chi_t$ be the mutation rate in every iteration $t\in \N$ and $\Delta_t h_2 = h_2(\chi_t) - h_2(\chi_{t+1})$ the difference of $h_2$ in iteration $t$.
        By definition, it holds that $\log_2 a > 0$ and $\log_2 b < 0$.
        For~$\chi_{t+1} = a \chi_t$ and event $B$, the difference $\Delta_t h_2$ is
        \begin{itemize}
            \item[-] $\Delta_t h_2 = c \log_2 a$; for $\chi_t \in [\chi_{min},\chi_{\text{opt}}/a]$,
            \item[-] $\Delta_t h_2 = -c \log_2 a$; for $\chi_t \in [\chi_{\text{opt}},\chi_{\max}]$,
            \item[-] $\Delta_t h_2\geq-c \log_2 a$; for $\chi_t \in [\chi_{\text{opt}}/a,\chi_{\text{opt}}]$.
        \end{itemize}
        Analogously, for $\chi_{t+1} = b \chi_t$ and event $\overline{B}$, the difference $\Delta_t h_2$ is
        \begin{itemize}
            \item[-] $\Delta_t h_2 = c \log_2 b$; for $\chi_t \in [\chi_{min},\chi_{\text{opt}}]$,
            \item[-] $\Delta_t h_2 = -c \log_2 b$; for $\chi_t \in [\chi_{\text{opt}}/b,\chi_{\max}]$,
            \item[-] $\Delta_t h_2 \geq c \log_2 b$; for $\chi_t \in [\chi_{\text{opt}},\chi_{\text{opt}}/b]$.
        \end{itemize}
        \label{claim:big-drift:change-g2}
    \end{claim}

    \begin{proof}[Proof of Claim~\ref{claim:big-drift:change-g2}]
        Suppose $\chi_t < \chi_{\text{opt}}/a$ and event $B$ in iteration $t \in \N$
        Following the event definition, the algorithm adjusts the mutation rate $\chi_t$ by multiplying it by $a$.
        Then, the difference of $h_2$ can be calculated by
        \begin{align*}
            &\Delta_t h_2\\
            &=h_2(\chi_t) - h_2(\chi_{t+1})\\
            &= h_2(\chi_t) - h_2(a\chi_t) \\
            &= c\left\lvert\log_2 \frac{\chi_t}{\chi_{\text{opt}}}\right\rvert - c \left\lvert\log_2 \frac{a\chi_t}{\chi_{\text{opt}}}\right\rvert\\
            &= c\Brackets{\log_2 \chi_{\text{opt}} - \log_2 \chi_t + \log_2 a + \log_2 \chi_t - \log_2 \chi_{\text{opt}}}\\
            &= c \log_2 a.
        \end{align*}
        The difference for $\chi_t \geq \chi_{\text{opt}}$ can be computed accordingly.
        Taking the absolute value of the logarithmic function in $h_2$ ensures that the difference $\Delta_t h_2$ for event $B$ and $\chi_t \in [\chi_{\text{opt}}/a,\chi_{\text{opt}}]$ is minimal for $\chi_t = \chi_{\text{opt}}$ which results in $\Delta_t h_2 \geq - c \log_2 a$.
        
        Using the same calculation and argumentation, the differences for event $\overline{B}$ follow.
    \end{proof}

    In order to bound the expected difference of $h_1$, we first show in the following claim, analogously to \cite[Theorem~3]{doerr2010multiplicative}, how to transform the multiplicative drift given for $g$ into an equivalent additive drift of $h_1$.

    \begin{claim}
        Let $\zeta \in \R$. 
        Suppose we obtain for $g$ a multiplicative drift of 
        \begin{align*}
            \ExMath{g(x_{t+1}) - g(x_t) \mid x_t} \geq \zeta \cdot g(x_t).
        \end{align*}
        Then, it is equivalent to an additive drift for $h_1$ of
        \begin{align*}
            \ExMath{h_1(x_{t+1}) - h_1(x_t) \mid h_1(x_t)} \geq \frac{\zeta}{2}.
        \end{align*}
        \label{claim:big-drift:equivalent-drift}
    \end{claim}

    \begin{proof}[Proof of Claim~\ref{claim:big-drift:equivalent-drift}]
        Using that $1 + \alpha \leq e^\alpha$ holds for all $\alpha \in \R$, it holds for all $z,z' \in \R_0^+$ that
        \begin{align*}
            \frac{1+z'}{1+z} = 1 + \frac{z'-z}{1+z} \leq \exp\Brackets{\frac{z'-z}{1+z}}.
        \end{align*}
        Taking the logarithm on both sides of the inequality and multiplying them by $-1$ leads to
        \begin{align*}
            \ln(1+z) - \ln(1+z') \geq \frac{z-z'}{1+z}.
        \end{align*}
        For all $t\in \N$, set $z = g(x_t)/g_{\min}$ and $z' = g(x_{t+1})/g_{\min}$ with~$g_{min}\geq~1$.
        Using $g_{min} \leq g(x_t)$ in the last inequality, results in
        \begin{align*}
            h_1(x_t) - h_1(x_{t+1}) &\geq \frac{g(x_t)/g_{\min} - g(x_{t+1})/g_{\min}}{1 + g(x_t)/g_{\min}}\\
            &= \frac{g(x_t) - g(x_{t+1})}{g_{\min} + g(x_t)}\\
            &\geq \frac{g(x_t) - g(x_{t+1})}{2 g(x_t)}.
        \end{align*}
        Suppose $t < T$ and $\ExMath{g(x_{t+1}) - g(x_t) \mid x_t} \geq \zeta \cdot g(x_t)$ as defined.
        Then,
        \begin{align*}
            \ExMath{h_1(x_{t+1}) - h_1(x_t) \mid h_1(x_t)} &\geq \frac{\ExMath{g(x_{t+1}) - g(x_t) \mid x_t}}{2 g(x_t)}\\
            &\geq \frac{\zeta \cdot g(x_t)}{2g(x_t)}\\
            &= \frac{\zeta}{2},
        \end{align*}
        which concludes the proof of Claim~\ref{claim:big-drift:equivalent-drift}.
    \end{proof}

    Now it is possible in the following claim to bound the expected drift of $h_1$ for different ranges of $\chi_t$ by bounding the expected drift of $g$ and then applying Claim~\ref{claim:big-drift:equivalent-drift}.

    \begin{claim}
        The expected difference $\ExMath{h_1(x_{t+1}) - h_1(x_t) \mid x_t}$ can be bounded for different ranges of the mutation rate $\chi_t$ as follows.
        
        \begin{enumerate}
            \item[(1)] Let $\chi_t \in \SqBrackets{\chi_{\min}, d \chi_{\text{opt}}}$.
                Then
                \begin{align}
                    \ExMath{h_1(x_{t+1}) - h_1(x_t) \mid x_t} \geq \frac{1}{4}\chi_t (1-\chi_t)^{2^r-1}. \label{eq:claim:lower-bound-drift-1}
                \end{align}
            \item[(2)] For $\chi_t \in \On(\chi_{\text{opt}})$, it holds that
                \begin{align}
                    \ExMath{h_1(x_{t+1}) - h_1(x_t) \mid x_t} \geq 0.\label{eq:claim:lower-bound-drift-2}
                \end{align}
            \item[(3)] Let $\chi_t \in \omega(\chi_{\text{opt}})$, $\chi_t \leq \frac{1}{2} = \chi_{\max}$ and $\delta_1 \in o(\chi_{\text{opt}})$.
                Then, the expected drift is
                \begin{align}
                    \ExMath{h_1(x_{t+1}) - h_1(x_t) \mid x_t} \geq - \delta_1.\label{eq:claim:lower-bound-drift-3}
                \end{align}
        \end{enumerate}
        \label{claim:big-drift:lower-bound-drift}
    \end{claim}

    \begin{proof}[Proof of Claim~\ref{claim:big-drift:lower-bound-drift}]
        We prove Claim~\ref{claim:big-drift:lower-bound-drift} analogously to \cite[Theorem~4.1]{WITT_2013} but handle the current mutation rate $\chi_t$ of the \EA independent of the mutation rate $d\chi_{\text{opt}}$ on which the definition of the potential function~$g$ is based.
        For the first part of the proof, we assume an arbitrary mutation rate $\chi_t \in \SqBrackets{\chi_{\min}, \chi_{\max}}$.
        
        Let $\Delta_t g := g(x_t) - g(x_{t+1})$ and fix an arbitrary current search point $x_t$.
        We denote $I := \CurlyBrackets{i \mid x_{t,i} = 1} \cap \CurlyBrackets{1,\ldots,2^r}$ as the index set of the $1$ bits in the first $2^r$ bits of $x_t$, $Z := \CurlyBrackets{1,\ldots,2^r}\setminus I$ the indices of the $0$ bits and $x_t'$ as the generated mutant in iteration~$t$.
        We define event $A$ such that $x_{t+1} = x_t' \neq x_t$ holds and that least one bit of the first $2^r$ bits differs in $x_{t+1}$ and $x_t$.
        If event $A$ does not occur, it holds that $\Delta_t g = 0$.
        Additionally to $I$ and $Z$, let $I^* := \CurlyBrackets{i \in I \mid x_{t,i}' = 0}$ and $Z^* := \CurlyBrackets{i \in Z \mid x_{t,i}' = 1}$ be the sets of $1$ and $0$ bits that flipped in the first $2^r$ bits of $x_t'$.
        In order that event $A$ occurs, $Z^* \neq \emptyset$ must hold.
        We divide the $1$ bits in $I$ into $L(i):= \CurlyBrackets{1,\ldots,i} \cap I$ and $R(i):= \CurlyBrackets{i+1,\ldots,2^r} \cap I$ as the bits on the left or right of the bit at index $i$.

        Let $i \in Z$ and define $A_i$ as the event such that
        \begin{enumerate}
            \item[(1)] $i$ is the leftmost flipping $0$ bit in the first $2^r$ bits (i.e. $i\in Z^*$ and $\CurlyBrackets{1,\ldots,i-1} \cap Z^* = \emptyset$), and
            \item[(2)] no $1$ bit on the left of $i$ flips (i.e. $I^* \cap L(i) = \emptyset$).
        \end{enumerate}
        Using this definition, the $A_i$ are mutually disjoint, and if none of the $A_i$ occurs, then $\Delta_t g = 0$ holds.

        For any $i \in Z$, $\Delta_t g$ is the sum of the two terms
        \begin{align}
            \Delta_L(i) := \sum_{j\in Z^*}g_j \geq g_i, \label{eq:claim:big-drift:lower-bounds-li}
        \end{align}
        and 
        \begin{align*}
            \Delta_R(i) := -\sum_{j\in I^*}g_j = -\sum_{j\in I^* \cap R(i)}g_j.
        \end{align*}
        Using the law of total probability and the linearity of expectation, the expected difference is
        \begin{align}
            \ExMath{\Delta_t g} &= \sum_{i \in Z} \ExMath{\Delta_t g_i}\notag\\
            &= \sum_{i \in Z} \ExMath{\Delta_L(i) \mid A_i} \cdot \PrMath{A_i} + \ExMath{\Delta_R(i) \mid A_i} \cdot \PrMath{A_i}. \label{eq:claim:big-drift:lower-bounds-delta}
        \end{align}
        In the following, we pessimistically assume that all bits in $R(i)$ flip from $1$ to $0$ independently with probability $\chi_t$, which results in
        \begin{align}
            \ExMath{\Delta_R(i) \mid A_i} &\geq -\chi_t \sum_{j \in R(i)} g_j\notag\\
            &\geq -\chi_t \sum_{j=i+1}^{2^r} g_j. \label{eq:claim:big-drift:lower-bounds-ri}
        \end{align}
        Since $A_i$ only occurs if $i$ is the leftmost bit flip in iteration $t$, it is equivalent to the probability that it is the only bit in the first $i$ bits that flips, which is given by
        \begin{align}
            \PrMath{A_i} = \chi_t(1-\chi_t)^{i-1}.\label{eq:claim:big-drift:lower-bounds-pi}
        \end{align}
        Using $i \leq 2^r$, it can be bounded by 
        \begin{align}
            \PrMath{A_i} \geq \chi_t(1-\chi_t)^{2^r-1}, \label{eq:claim:big-drift:lower-bounds-pr_ai_1}
        \end{align}
        and 
        \begin{align}
            \PrMath{A_i} \leq \chi_t. \label{eq:claim:big-drift:lower-bounds-pr_ai_2}
        \end{align}
        
        Using (\ref{eq:claim:big-drift:lower-bounds-li}) and (\ref{eq:claim:big-drift:lower-bounds-ri}) in (\ref{eq:claim:big-drift:lower-bounds-delta}), it follows that
        \begin{align}
            \ExMath{\Delta_t g} &\geq \sum_{i \in Z} \Brackets{\PrMath{A_i} g_i - \PrMath{A_i}\chi_t \sum_{j=i+1}^{2^r} g_j}. \label{eq:claim:big-drift:lower-bounds-g_1}
        \end{align}
    
        Using the finite linear sum, for all $i \in Z$, the term in the sum can be bounded by
        \begin{align}
            &\PrMath{A_i} g_i - \PrMath{A_i}\chi_t \sum_{j=i+1}^{2^r} g_j\notag \\
            &\geq \PrMath{A_i} g_i - \PrMath{A_i} \chi_t \frac{g_i-1}{\Brackets{\frac{\alpha d \chi_{\text{opt}}}{(1-d\chi_{\text{opt}})^{2^r-1}}}}\notag\\
            &\geq \PrMath{A_i} g_i - \PrMath{A_i} \chi_t \frac{g_i(1-d\chi_{\text{opt}})^{2^r-1}}{\alpha d \chi_{\text{opt}}}. \label{eq:claim:big-drift:lower-bounds-i_1}
        \end{align}

        We distinct now between the different ranges of $\chi_t$ as defined in Claim~\ref{claim:big-drift:lower-bound-drift}.

        First, let $\chi_t \in \SqBrackets{\chi_{\min}, d\chi_{\text{opt}}}$.
        Using the lower bound (\ref{eq:claim:big-drift:lower-bounds-pr_ai_1}) in the first term of (\ref{eq:claim:big-drift:lower-bounds-i_1}) and the bounds $\chi_t \leq d\chi_{\text{opt}}$ and (\ref{eq:claim:big-drift:lower-bounds-pr_ai_2}) in the second term of (\ref{eq:claim:big-drift:lower-bounds-i_1}), results in
        \begin{align}
            \ExMath{\Delta_t g_i} &\geq \chi_t (1-\chi_t)^{2^r-1} g_i - \chi_t (1-\chi_t)^{2^r-1} \frac{g_i}{\alpha}\notag\\
            &\geq \chi_t (1-\chi_t)^{2^r-1} \Brackets{1-\frac{1}{\alpha}} g_i.\label{eq:claim:big-drift:lower-bounds-i_2}
        \end{align}
        Inserting (\ref{eq:claim:big-drift:lower-bounds-i_2}) into (\ref{eq:claim:big-drift:lower-bounds-g_1}), leads to
        \begin{align}
            \ExMath{\Delta_t g} &\geq \sum_{i \in Z} \chi_t (1-\chi_t)^{2^r-1} \Brackets{1-\frac{1}{\alpha}} g_i\notag\\
            &= \chi_t (1-\chi_t)^{2^r-1} \Brackets{1-\frac{1}{\alpha}} g(x_t).\label{eq:claim:big-drift:lower-bounds-g_2}
        \end{align}
        Applying Claim~\ref{claim:big-drift:equivalent-drift} to (\ref{eq:claim:big-drift:lower-bounds-g_2}) and using $\alpha > 2$, results in
        \begin{align*}
            \ExMath{h_1(x_{t+1}) - h_1(x_t) \mid x_t} \geq \frac{1}{4}\chi_t (1-\chi_t)^{2^r-1},
        \end{align*}
        which concludes the proof of the first part of Claim~\ref{claim:big-drift:lower-bound-drift}.

        For the remaining cases of $\chi_{t}$, we insert the sharper bound (\ref{eq:claim:big-drift:lower-bounds-pi}) in (\ref{eq:claim:big-drift:lower-bounds-i_1}).
        Using $(1-d\chi_{\text{opt}})^{2^r-1} \leq 2/e^{d}$ in the second inequality, results in
        \begin{align}
             \ExMath{\Delta_t g_i} &\geq \chi_t(1-\chi_t)^{i-1} g_i -  \frac{\chi_t^2(1-\chi_t)^{i-1}g_i (1-d\chi_{\text{opt}})^{2^r-1}}{\alpha d \chi_{\text{opt}}}\notag\\
             &\geq \chi_t(1-\chi_t)^{i-1} g_i \Brackets{1- \frac{2\chi_t}{\alpha e^d d \chi_{\text{opt}}}}. \label{eq:claim:big-drift:lower-bounds-i_3}
        \end{align}

        For $\chi_t \in \On\Brackets{\chi_{\text{opt}}}$, let $d' > d$ with $\chi_t = d'\chi_{\text{opt}}$.
        Since we can choose $\alpha \in \Theta(1)$ arbitrary large, there exists an $\alpha$ such that $d'~<~d\alpha e^d$.
        Then, it holds for (\ref{eq:claim:big-drift:lower-bounds-i_3}) that
        \begin{align}
            \ExMath{\Delta_t g_i} &\geq 0. \label{eq:claim:big-drift:lower-bounds-i_4}
        \end{align}
        Inserting (\ref{eq:claim:big-drift:lower-bounds-i_4}) into (\ref{eq:claim:big-drift:lower-bounds-g_1}) results in $\ExMath{\Delta_t g} \geq 0$ and applying Claim~\ref{claim:big-drift:equivalent-drift} preserves this.
        Therefore, the second statement of Claim~\ref{claim:big-drift:lower-bound-drift} follows.

        In the third case, for $\chi_t \in \omega(\chi_{\text{opt}})$ it holds that the last factor in~(\ref{eq:claim:big-drift:lower-bounds-i_3}) can become negative but is bounded for $\chi_t \leq \frac{1}{2}$ by
        \begin{align*}
            \Brackets{1- \frac{2\chi_t}{\alpha e^d d \chi_{\text{opt}}}} \leq \Brackets{1- \frac{1}{\alpha e^d d x_{\text{opt}}}}.
        \end{align*}
        However, for larger $\chi_t \in \omega(\chi_{\text{opt}})$, the probability $\chi_t(1-\chi_t)^{i-1}$ of event $A_i$ decreases exponentially and dominates the potential negative factor, such that it holds for the absolute difference that
        \begin{align}
            \ExMath{\lvert \Delta_t g_i \rvert} \in o\Brackets{\chi_{\text{opt}}} \label{eq:claim:big-drift:lower-bounds-i_5}.
        \end{align}
        Using our phase definition and \cite[Theorem~1]{jagerskupper2008blend}, at most $\gamma 2^{r-1}$ bits of value $0$ are in the first $2^{r-1}$ bits and are less likely on more significant indices.
        Therefore, the bound in (\ref{eq:claim:big-drift:lower-bounds-i_5}) remains after taking the sum over all possible $i \in Z$ and after applying Claim~\ref{claim:big-drift:equivalent-drift}, which concludes the proof of Claim~\ref{claim:big-drift:lower-bound-drift}. 
    \end{proof}

    For the following case distinction, we need upper and lower bounds for the distribution of the $0$ and $1$ bits at the indices larger than $2^r$.
    To obtain this, the next observation follows directly from~\cite[Theorem~1]{jagerskupper2008blend}.

    \begin{observation}
        Let $X_{t,i}$ be a random variable that denotes the bit value of the bit at index $i~\in~\{1, \ldots, n\}$ in iteration $t~\in~\N$.
        Assume for~$i = 2^r$ that $\PrMath{X_{t,i} = 1} \leq (1-\gamma)$ holds.
        Let $j \in \{2^r+1,\ldots,n\}$.
        Then, $\PrMath{X_{t,j} = 1} \leq (1-\gamma)$ and $\PrMath{X_{t,j} = 0} \geq \gamma$  holds for any bit with index $j$.
        \label{cor_jägersküpper}
    \end{observation}
    
    We can apply Observation~\ref{cor_jägersküpper} to our phase definition since it ensures that during Phase~$r$, at least $\gamma 2^r$ bits of the first $2^r$ bits are set to~$0$.
    The probability $\PrMath{X_{t,i} = 1} \leq (1-\gamma)$ for $i = 2^r$ follows according to~\cite{jagerskupper2008blend}.
    
    In the following, we prove the lower bound (\ref{eq_drift_condition_cases}) of the expected combined drift of the potential function $h$ for the defined three cases.\\

   \textbf{Case I: $\chi_t \in [\chi_{\min}, \chi_{\text{opt}}/d]$}

    The intuition behind the proof for $\chi_t \in [\chi_{\min}, \chi_{\text{opt}}/d]$ is that we consider the mutation rate as too small to make progress in the first $2^r$ bits that are decisive for the current Phase~$r$, but the bit flips of less significant bits ensure that the mutation rate increases.

    Since the potential function is additively defined, we divide the analysis of the expected drift into the two parts $h_1$ and $h_2$.
    The second part $h_2$ can be further split according to the intersection of the events $A$ and $B$ and their negations.
    Using Claim~\ref{claim:big-drift:lower-bound-drift}, we lower bound the expected drift of $h_1$ by $0$ in the first inequality.
    As the event $A\cap B$ (leftmost bit flip is a $0$-to-$1$ flip in the first $2^r$ bits) occurs with the same probability as a change of $h_1$ and increases the value of $h_2$, we also lower bound it by $0$ in the first inequality.
    For the remaining events, we use the calculated differences in Claim~\ref{claim:big-drift:change-g2} and the definition of conditional probabilities in the second equality.
    Then, the expected drift in Case~I can be lower bounded by
    \begin{align}
        &\ExMath{\Delta_t h\mid x_t, \chi_t} \notag\\
        &= \ExMath{\Delta_t h_1\mid x_t} + \ExMath{\Delta_t h_2\mid \chi_t}\notag\\
        &= \ExMath{\Delta_t h_1 \mid x_t} + c \cdot \PrMath{A\cap B} (\log_2 a)\notag\\
        &+ c \cdot \PrMath{A\cap \overline{B}} (\log_2 b)\notag\\
        &+ c \cdot \PrMath{\overline{A}}\left(\PrMath{B\mid \overline{A}} (\log_2 a) + \PrMath{\overline{B} \mid \overline{A}} (\log_2 b) \right)\notag\\
        &\geq c \cdot \PrMath{A\cap \overline{B}} (\log_2 b) \label{eq_drift_case1_line_1} \\
        &+ c \cdot \PrMath{\overline{A}}\left(\PrMath{B\mid \overline{A}} (\log_2 a) + \PrMath{\overline{B} \mid \overline{A}} (\log_2 b) \right) \label{eq_drift_case1_line_2}
    \end{align}
    with $\log_2 a > 0$ and $\log_2 b < 0$.
    Since event $A$ decides in which index range the leftmost bit flip occurs, it is possible to divide the expected difference into (\ref{eq_drift_case1_line_1}) and (\ref{eq_drift_case1_line_2}), covering the remaining combinations of the events $A$ and $B$.
    
    The expected difference contributed by event $A \cap \overline{B}$ in (\ref{eq_drift_case1_line_1}) is negative.
    Its probability $\PrMath{A \cap \overline{B}}$ can be upper bounded by the probability of event $A$, which is equal to the complementary probability that none of the first $2^r$ bits flips, and we compute its upper bound by using the maximum value of $\chi_t = \chi_{\text{opt}}/d$ in Case~I.
    In the second inequality of the calculation, we apply Bernoulli's inequality.
    We derive an upper bound for the probability of a bit flip in the first~$2^r$ bits by
    \begin{align}
        \PrMath{A} &= 1 - (1-\chi_t)^{2^r} \label{eq_ub_a_first_line}\\
        &\leq 1 - \left(1 - \frac{\chi_{\text{opt}}}{d}\right)^{2^r} \label{eq_ub_a_second_line}\\
        &\leq 1 - \left(1 - 2^r\frac{\chi_{\text{opt}}}{d}\right)\label{eq_ub_a_third_line}\\
        &= 1 - \left(1 - \frac{1}{d}\right)\notag\\
        &= \frac{1}{d}\notag.
    \end{align}
    Then, the expected negative drift in (\ref{eq_drift_case1_line_1}) can be lower bounded by
    \begin{align*}
        c \cdot \PrMath{A\cap \overline{B}} (\log_2 b) &\geq c \cdot \PrMath{A} (\log_2 b)\\
        &\geq c \frac{1}{d}(\log_2 b).
    \end{align*}

    To fulfill Equation (\ref{eq_drift_condition_cases}), (\ref{eq_drift_case1_line_2}) needs to be positive.
    We derive its lower bound by taking the lower bound of 
    \begin{align*}
        \PrMath{\overline{A}} = 1 - \PrMath{A}\geq \frac{d-1}{d},
    \end{align*}
    and calculating $\PrMath{B\mid \overline{A}}$ and $\PrMath{\overline{B} \mid \overline{A}}$.
    Following the definition of the events $A$ and $B$, the conditional probabilities are the probabilities of whether the bit of the leftmost bit flip after index $2^r$ is a $0$ or~$1$ bit.
    Using Observation~\ref{cor_jägersküpper}, we can bound them by $\PrMath{B\mid \overline{A}} \geq \gamma$ and $\PrMath{\overline{B} \mid \overline{A}} \leq (1-\gamma)$.
    Combining the results of the lower bounds of (\ref{eq_drift_case1_line_1}) and (\ref{eq_drift_case1_line_2}) leads to
    \begin{align}
        \ExMath{\Delta_t h\mid x_t, \chi_t} &\geq c \left ( \frac{1}{d}\log_2 b + \frac{d-1}{d} (\gamma \log_2 a + (1-\gamma) \log_2 b ) \right ). \label{eq_case_1_final_inequality}
    \end{align}
    For an overall positive expected drift in Case I, it must hold that
    \begin{align*}
        \frac{1}{d}\log_2 b + \frac{d-1}{d} (\gamma \log_2 a + (1-\gamma) \log_2 b ) &> 0\\
        \Leftrightarrow \frac{\log_2 a}{-\log_2 b} &> \frac{1-\gamma}{\gamma} + \frac{1}{(d-1)\gamma},
    \end{align*}
    which is fulfilled by our assumption on $a$, $b$, $d$, and $\gamma$.
    Since it depends only on these constants, the difference of the terms in the Inequality~(\ref{eq_case_1_final_inequality}) is also a positive constant.
    Using $c\in \Theta(1/2^r)$ results in $\ExMath{\Delta_t h\mid x_t, \chi_t} \geq \delta$ with $\delta \in \Theta(1/2^r)$.\\

    \textbf{Case II: $\chi_t \in [\chi_{\text{opt}}/d, d\chi_{\text{opt}}]$}

    In order to derive the lower bound for the expected drift in the case of a mutation rate $\chi_t \in [\chi_{\text{opt}}/d, d\chi_{\text{opt}}]$, we show that the algorithm is able to make progress in the first $2^r$ bits.
    Using the additive property of the potential function and the lower bounds of Claim~\ref{claim:big-drift:change-g2} in the first inequality, we get an expected additive drift of 
    \begin{align}
        &\ExMath{\Delta_t h \mid x_t,\chi_t} \notag \\
        &= \ExMath{\Delta_t h_1\mid x_t} + \ExMath{\Delta_t h_2\mid \chi_t} \notag\\
        &\geq \ExMath{\Delta_t h_1\mid x_t} + c \cdot \Brackets{\PrMath{B} \cdot (- \log_2 a) + \PrMath{\overline{B}} \cdot \log_2 b} \label{eq:drift:case2:init-lower-bound}
    \end{align}
    with $\log_2 a > 0$ and $\log_2 b < 0$.

    Using Claim~\ref{claim:big-drift:lower-bound-drift}, the expected drift of $h_1$ for $\chi_t \leq d\chi_{\text{opt}}$ is 
    \begin{align*}
        \ExMath{\Delta_t h_1\mid x_t} \geq \frac{1}{4}\chi_t (1-\chi_t)^{2^r-1}.
    \end{align*}
    Since $\chi_t~=~\chi_{\text{opt}}$ is the only local and global maximum of this function, we calculate the expected additive drift for the borders of the range $[\chi_{\text{opt}}/d, d\chi_{\text{opt}}]$ and show that Equation~(\ref{eq_drift_condition_cases}) holds for both of them.
    
    Let $\chi_t = \chi_{\text{opt}}/d$.
    Using Bernoulli's inequality in the third inequality, the expected difference can be lower-bounded by 
    \begin{align}
        \ExMath{\Delta_t h_1\mid x_t} &\geq \frac{1}{4}\chi_t (1-\chi_t)^{2^r-1}\notag \\
        &\geq \frac{1}{4}\frac{\chi_{\text{opt}}}{d}\left(1-\frac{\chi_{\text{opt}}}{d}\right)^{2^{r-1}}\notag\\
        &\geq \frac{1}{4}\frac{1}{d2^r}\left(1 - \frac{1}{2d}\right) \label{eq_ub_2_lower_border_line3}\\
        &= \frac{1}{4}\frac{1}{2^r}\frac{2d-1}{2d^2}\notag\\
        &=\zeta_1.\notag
    \end{align}
    
    For $\chi_t = d \chi_{\text{opt}}$, we can calculate the lower bound analogously.
    Using $(1-d\chi_{\text{opt}})^{2^{r-1}}~\geq~e^{-d}$ in the third inequality, the expected drift is lower bounded by
    \begin{align}
        \ExMath{\Delta_t h_1\mid x_t} &\geq \frac{1}{4}\chi_t(1-\chi_t)^{2^{r-1}} \notag\\
        &\geq \frac{1}{4}d\chi_{\text{opt}}(1-d\chi_{\text{opt}})^{2^{r-1}}\notag\\
        &\geq \frac{1}{4}\frac{1}{2^r}\frac{d}{e^d} \label{eq_ub_2_upper_border_line3}\\
        &=\zeta_2.\notag
    \end{align}
    
    Let $\zeta = \min\{\zeta_1, \zeta_2\} = \frac{1}{4} \min\left\{\frac{2d-1}{2^rd^2}, \frac{d}{2^re^d}\right\} \in \Theta(1/2^r)$.
    Then, the expected drift of $h_1$ is lower bounded by
    \begin{align*}
        \ExMath{\Delta_t h_1\mid x_t} \geq \zeta \in \Theta(1/2^r).
    \end{align*}
    By our assumption on $c$, it holds that $\zeta > c \cdot \max\CurlyBrackets{\log_2 a, -\log_2 b}$, which ensures that the potential negative drift of $h_2$ in (\ref{eq:drift:case2:init-lower-bound}) does not eliminate the positive drift of $h_1$.
    Then, $\ExMath{\Delta_t h\mid Z_t, \chi_t} \geq \delta$ with $\delta \in \Theta(1/2^r)$ follows.\\

    \textbf{Case III: $\chi_t \in [d\chi, \chi_{\max}]$}
    
    The intuition behind the analysis of Case~III is that it becomes too unlikely to make progress in the first $2^r$ bits with a mutation rate in the range of $\chi_t \in [d \chi_{\text{opt}}, \chi_{\max}]$, since it becomes too likely that the algorithm flips one of the $1$ bits in the first $2^r$ bits as the leftmost bit.
    However, these rejections ensure that the algorithm decreases the mutation rate towards its optimal range.
    
    Since the potential function is additively defined, we divide the analysis of the expected drift into the two parts $h_1$ and $h_2$, analogously to the analysis of Case~I.
    Again, the second part $h_2$ can be further split according to our defined events $A$ and $B$ and their negations.
    Using Claim~\ref{claim:big-drift:lower-bound-drift}, it is possible to lower bound the expected drift of $h_1$ by a small negative value in the order of $o(1/2^r)$.
    Since event $A \cap B$ occurs with the same exponentially decreasing probability and contributes a constant negative difference of $\log_2 a$, we can bound both expected differences by $-\delta_1$ with $\delta_1 \in o(1/2^r)$ in the first inequality.
    For the remaining events, we use the calculated differences in Claim~\ref{claim:big-drift:change-g2} and the definition of conditional probabilities in the first equality.
    Then, the expected drift in Case~III can be lower bounded by
    \begin{align}
        &\ExMath{\Delta_t h\mid x_t, \chi_t} \notag\\
        &= \ExMath{\Delta_t h_1\mid x_t} + \ExMath{\Delta_t h_2\mid \chi_t}\notag\\
        &= \ExMath{\Delta_t h_1\mid x_t} + c \cdot \PrMath{A\cap B} (-\log_2 a)\notag\\
        &+ c \cdot \PrMath{A\cap \overline{B}} (-\log_2 b)\notag\\
        &+ c \cdot \PrMath{\overline{A}}\left(\PrMath{B\mid \overline{A}} (-\log_2 a) + \PrMath{\overline{B} \mid \overline{A}} (-\log_2 b) \right)\notag\\
        &\geq -\delta_1\notag \\
        &+ c \cdot \PrMath{A\cap \overline{B}} (-\log_2 b) \label{eq_drift_case3_line_1}\\
        &+ c \cdot \PrMath{\overline{A}}\left(\PrMath{B\mid \overline{A}} (-\log_2 a) + \PrMath{\overline{B} \mid \overline{A}} (-\log_2 b) \right) \label{eq_drift_case3_line_2}
    \end{align}
    with $\log_2 a > 0$, $\log_2 b < 0$ and $\delta_1 \in o(1/2^r)$.
    Since event $A$ decides in which index range the leftmost bit flip occurs, it is possible to divide the expected drift into (\ref{eq_drift_case3_line_1}) and (\ref{eq_drift_case3_line_2}), covering the remaining combinations of the events $A$ and $B$.
    
    We get the probability of event $A \cap \overline{B}$ in (\ref{eq_drift_case3_line_1}) by
    \begin{align*}
        \PrMath{A\cap \overline{B}} = \PrMath{A}\PrMath{\overline{B} \mid A}
    \end{align*}
    To calculate the probability of event $\overline{B}$ under the constraint $A$, we first analyze the distribution of $0$ and $1$ bits in the first $2^r$ bits in Phase~$r$.
    According to the phase definition, at least $(1-\gamma)$ of the first $2^{r-1}$ bits are $1$ bits.
    The initialization uniformly at random and that $0$-to-$1$ flips are at least as likely as $1$-to-$0$ flips according to~\cite{jagerskupper2008blend} implies that at least half of the next $2^{r-1}$ bits are $1$ bits.
    Therefore, the probability of event $\overline{B}$ in the first $2^r$ bits can be lower bounded in Phase~$r$ by $\PrMath{\overline{B} \mid A} \geq \left(\frac{1}{2}(1-\gamma) + \frac{1}{2}\frac{1}{2}\right )= \left(\frac{3}{4} - \frac{\gamma}{2} \right)$.

    The probability of event $A$ is equal to the complementary probability that none of the first $2^r$ bits flip.
    To calculate its lower bound in this case, we set $\chi_t$ to the lower border of the range $[d \chi_{\text{opt}}, \chi_{\max}]$ and apply Bernoulli's inequality in the second inequality of the calculation.
    Then, we derive a lower bound of the probability that at least one bit in the first $2^r$ bits flips by
    \begin{align*}
        \PrMath{A} &= 1 - (1-\chi_t)^{2^r} \notag\\
        &\geq 1 - (1 - d \chi_{\text{opt}})^{2^r}\\
        &\geq 1 - \frac{1}{d+1}\\
        &= \frac{d}{d+1}.
    \end{align*}
    The lower bound $\PrMath{A\cap \overline{B}} = \PrMath{A}\PrMath{\overline{B} \mid A} \geq \frac{d}{d+1}\left(\frac{3}{4} - \frac{\gamma}{2} \right)$ follows.

    Since the expected drift induced by event $\overline{A}$ in (\ref{eq_drift_case3_line_2}) is negative, we can use the upper bound $\PrMath{\overline{A}} = 1 - \PrMath{A} \leq \frac{1}{d+1}$.
    Again, according to the initialization uniformly at random and~\cite{jagerskupper2008blend}, we know that at least half of the bits after index $2^r$ are $1$ bits and at most the other half are $0$ bits.
    The respective conditional probabilities $\PrMath{\overline{B} \mid \overline{A}} \geq \frac{1}{2}$ and $\PrMath{B\mid \overline{A}} \leq \frac{1}{2}$ follow.

    Using the lower bounds for (\ref{eq_drift_case3_line_1}) and (\ref{eq_drift_case3_line_2}), the expected drift can be lower bounded by
    \begin{align}
        &\ExMath{\Delta_t h\mid x_t, \chi_t}\notag\\
        &\geq -\delta_1 \notag\\
        &+ c \left ( \frac{d}{d+1}\left(\frac{3}{4} - \frac{\gamma}{2} \right)(-\log_2 b)+ \frac{1}{d+1} \left(\frac{1}{2} (-\log_2 a) + \frac{1}{2} (-\log_2 b) \right) \right ). 
        \label{eq_case_3_final_inequality}
    \end{align}
    For an overall positive expected drift in Case III, it must hold that
    \begin{align*}
        &0 < \frac{d}{d+1}\left(\frac{3}{4} - \frac{\gamma}{2} \right)(-\log_2 b) + \frac{1}{d+1} \left(\frac{1}{2} (-\log_2 a) + \frac{1}{2} (-\log_2 b) \right)\\
        &\Leftrightarrow \frac{\log_2 a}{-\log_2 b} < 2d \left(\frac{3}{4} - \frac{\gamma}{2} \right) + 1,
    \end{align*}
    which is fulfilled by our assumption on $a$, $b$, $d$, and $\gamma$.
    Since it depends only on these constants, the difference of the terms in the Inequality~(\ref{eq_case_3_final_inequality}) is also a positive constant.
    Using $\delta_1 \in o(1/2^r)$ and $c\in \Theta(1/2^r)$, it results in $\ExMath{\Delta_t h\mid x_t, \chi_t} \geq \delta$ with $\delta \in \Theta(1/2^r)$, which concludes the proof of \Cref{lem_self_adj_mut_phase_analysis}.
\end{proofE}

After proving in the previous section the expected run time of singular phases of \Cref{alg_ea_self_adjust_mr}, it is now possible to show the expected run time of optimizing the first $k \in o(n)$ bits of the bit string for all $k$ simultaneously.

\begin{theoremE}[][normal]
    Let $n \in \mathbb{N}$ be the length of the bit string $x$ that gets optimized and $k \in o(n)$. 
    Consider the \EA with self-adjusting mutation rate that optimizes \BinVal.
    Let $0 < \gamma < 1 - \ln(2)$ be a constant and let $a > 0$ and $0 < b < 1$ be the parameters of the algorithm such that $\log_2 a/(-\log_2 b) > (1-\gamma)/\gamma$ (one can choose, for example,~$\gamma = 0.15$, $a=1.85$ and $b=0.898$).
    Then, the expected number of iterations until the first $k$ bits are set to $1$ is in $\OMath{k^{1+1/\log(1/\gamma)}\log k}$.
    \label{theo_runtime_self_adj_mr}
\end{theoremE}

\begin{proof}[Proof sketch]
    The intuition behind the proof of \Cref{theo_runtime_self_adj_mr} is that, after $\log_2 k$ phases, $\gamma k$ bits of value $0$ are left in the first $k$ bits, but this amount gets reduced by a factor of $(1-\gamma)$ in each following phase due to the more than doubled amount of iterations and halved optimal mutation rate in each following phase.
    To bound the number of iterations until the number of remaining $0$ bits in the first $k$ bits is zero, we first use the tail bounds of the multiplicative drift theorem in \cite{doerr2010tailorbounds} to get the number of phases it takes and then use Wald's equation to retrieve the expected number of iterations.
\end{proof}
    
\begin{proofE}
    Let $T_1 = \log_2 k$ denote the number of phases until $\gamma k$ of the first $k$ bits are set to $1$ and $T_2$ the number of phases after $T_1$ until the first $k$ bits are fully optimized.
    Let $T$ describe the number of iterations of \Cref{alg_ea_self_adjust_mr} until the first $k$ bits are set to $1$, and $Q_i$ with $\ExMath{Q_i} = 2^i \log_2 2^i$ the number of iterations in Phase $i$.
    Then,
    \begin{align}
        \ExMath{T} &= \ExMath{\sum_{i=0}^{T_1}Q_i + \sum_{i=1}^{T_2}Q_{\log_2 k + i}}\notag\\
        &= \ExMath{\sum_{i=0}^{T_1}Q_i} + \ExMath{\sum_{i=1}^{T_2}Q_{\log_2 k + i}}\notag\\
        &= \sum_{i=0}^{\log_2 k}\ExMath{Q_i} + \ExMath{\sum_{i=1}^{T_2}Q_{\log_2 k + i}}\notag\\
        &= \sum_{i=0}^{\log_2 k}2^i \log_2 2^i + \ExMath{\sum_{i=1}^{T_2}Q_{\log_2 k + i}}
        \label{eq_anytime_expected_t}
    \end{align}

    Since the largest summand dominates the sum, for the first term, it holds that
    \begin{align}
        \sum_{i=0}^{\log_2 k}2^i \log_2 2^i &\in \OMath{2^{\log_2 k}\log 2^{\log_2 k}}\notag\\
        &\in \OMath{k \log k}.
        \label{eq_anytime_bound_t1}
    \end{align}

    Using the independence of $T_2$ and $Q_i$, it is possible to apply Wald's equation to bound the second term by
    \begin{align}
        \ExMath{\sum_{i=1}^{T_2}Q_{\log_2 k + i}} &\leq \sum_{i = 1}^{\infty} \ExMath{Q_{\log_2 k + i} \mid 1_{\{T_2 \geq i\}}} \notag\\
        &= \sum_{i = 1}^{\infty}\ExMath{Q_{\log_2 k + i}} \PrMath{T_2 \geq i}.
        \label{eq_anytime_walds_equation}
    \end{align}

    In order to calculate the probabilities $\PrMath{T_2 \geq i}$, we use the tail bounds of the multiplicative drift as proven in~\cite{doerr2010tailorbounds}.
    Since we use the multiplicative drift argument to get the expected number of phases, $t \in \N$ describes in the following the number of phases and not iterations.
    For all $t \in \N$, we let $Z_t$ denote the number of~$0$~bits in the first $k$ bits in every phase $t$.
    Using \Cref{lem_self_adj_mut_phase_analysis}, we know the runtime until only $\gamma k$ bits of value $0$ in the first $k$ bits remain.
    Therefore, we start the analysis with Phase~$\log_2 k$ as $t=0$ and $\ExMath{Z_0} = \gamma k$ follows.
    Let $X_t = 2^{\log_2 k + t}$ denote the number of bits the algorithm optimizes in every Phase~$t \in \N$.
    In comparison to Phase~$t$, we optimize in Phase~$t+1$ the doubled amount of bits $X_{t+1} = 2 X_t$ according to the phase definition.
    Using \Cref{lem_self_adj_mut_phase_analysis}, it takes $\On(X_{t+1} \log X_{t+1}) = \On(2X_t \log 2X_t)$ iterations in this phase until $(1-\gamma) X_{t+1}$ of the first $X_{t+1}$  bits are set to $1$.
    The target optimal mutation rate of Phase $X_{t+1}$ is $1/X_{t+1} = 1/(2X_t)$ as proven in \cite[Theorem~3.1]{WITT_2013}.
    For the remaining bits in the first $k$ bits, it holds that the algorithm optimizes in every phase after Phase $\log_2 k$ in comparison to the last phase the doubled amount of bits in more than the doubled amount of iterations, which cancels out the halved optimal mutation rate and preserves the rate of optimizing $(1-\gamma)$ of the first $k$ bits in Phase $\log_2 k$.
    Therefore, the expected difference can be lower-bounded by
    \begin{align}
        \ExMath{Z_t - Z_{t+1} \mid Z_t} &\geq Z_t - \gamma Z_t\notag\\
        &= (1-\gamma)Z_t\notag\\ 
        &= \delta Z_t
        \label{eq_anytime_delta}
    \end{align}
    with $\delta = (1-\gamma)$.
    According to~\cite{doerr2010tailorbounds}, the tail bounds of $T_2$ can be calculated by
    \begin{align}
        \PrMath{T_2 > \frac{i + \ln \gamma k}{\delta}\mid Z_0 \leq \gamma k} &= \PrMath{T_2 > \frac{i + \ln \gamma k}{(1-\gamma)}\mid Z_0 \leq \gamma k}\notag\\
        &\leq e^{-i}.
        \label{eq_anytime_tailor_bounds}
    \end{align}

    Applying (\ref{eq_anytime_tailor_bounds}) to (\ref{eq_anytime_walds_equation}), results in
    \begin{align}
        &\sum_{i = 1}^{\infty} \ExMath{Q_{\log_2 k + i}} \PrMath{T_2 \geq i} \notag\\
        &\leq \sum_{i = 0}^{\infty} \ExMath{Q_{\log_2 k + \frac{i + \ln \gamma k}{(1-\gamma)}}} \PrMath{T_2 > \frac{i + \ln \gamma k}{(1-\gamma)}\mid Z_0 \leq \gamma k} \notag\\
        &\leq \sum_{i = 0}^{\infty} \left[ \left (2^{\log_2 k + \frac{i + \ln \gamma k}{(1-\gamma)}} \log_2 2^{\log_2 k + \frac{i + \ln \gamma k}{(1-\gamma)}} \right)  e^{-i}\right]\notag\\
        &=k \sum_{i = 0}^{\infty} \left[ 2^{\frac{i + \ln \gamma k}{(1-\gamma)}} \Brackets{\log_2 k + \frac{i + \ln \gamma k}{(1-\gamma)}}   e^{-i}\right]\notag\\
        &= k 2^{\frac{\ln \gamma k}{(1-\gamma)}} \sum_{i = 0}^{\infty} \left [ \left(2^{\frac{i}{(1-\gamma)}} \right) \Brackets{\log_2 k + \frac{i}{1-\gamma} + \frac{\ln \gamma k}{(1-\gamma)}} e^{-i}\right]\notag\\
        &= k 2^{\frac{\ln \gamma k}{(1-\gamma)}} \log_2 k \sum_{i = 0}^{\infty} \Brackets{\frac{2^{\frac{1}{(1-\gamma)}}}{e}}^i \notag\\
        &+ k 2^{\frac{\ln \gamma k}{(1-\gamma)}}\frac{1}{1-\gamma} \sum_{i = 0}^{\infty} i\Brackets{\frac{2^{\frac{1}{(1-\gamma)}}}{e}}^i \notag\\
        &+ k 2^{\frac{\ln \gamma k}{(1-\gamma)}}\frac{\ln \gamma k}{(1-\gamma)}\sum_{i = 0}^{\infty} \Brackets{\frac{2^{\frac{1}{(1-\gamma)}}}{e}}^i.
        \label{eq_anytime_calculation_t2}
    \end{align}

    For $\gamma < 1 - \ln 2$ it holds that
    \begin{align*}
        \frac{2^{\frac{1}{(1-\gamma)}}}{e} &< \frac{2^{\frac{1}{\ln 2}}}{e}\\
        &\leq \frac{1}{\ln2} \ln 2\\
        &= 1.
    \end{align*}
    Let 
    \begin{align}
        r =\frac{2^{\frac{1}{(1-\gamma)}}}{e} < 1.
        \label{eq_def_r}
    \end{align}
    Then, applying the geometric series formula results in the constants
    \begin{align}
        \sum_{i = 0}^{\infty} \left[ \left(\frac{2^{\frac{1}{(1-\gamma)}}}{e} \right)^i \right] &= \sum_{i = 0}^{\infty} r ^i\notag\\
        &= \frac{1}{1-r},
        \label{eq_anytime_geom_series_1}
    \end{align}
    and
    \begin{align}
        \sum_{i = 0}^{\infty} \left[i \left(\frac{2^{\frac{1}{(1-\gamma)}}}{e} \right)^i \right] &= \sum_{i = 0}^{\infty} i r ^i\notag\\
        &= \frac{r}{(1-r)^2}.
        \label{eq_anytime_geom_series_2}
    \end{align}

    Applying logarithmic base transformations leads to
    \begin{align}
        2^{\frac{\ln \gamma k}{(1-\gamma)}} &\in \OMath{2^{\log_{1/\gamma}k}}\notag\\
        &\in \OMath{k^{\log_{1/\gamma}2}}\notag\\
        &\in \OMath{k^{1/(\log_2 (1/\gamma)}}.
        \label{eq:anytime:log_transf}
    \end{align}

    Using (\ref{eq_def_r}), the constants (\ref{eq_anytime_geom_series_1}) and (\ref{eq_anytime_geom_series_2}), and the result of (\ref{eq:anytime:log_transf}) in (\ref{eq_anytime_calculation_t2}), leads to
    \begin{align}
        &k 2^{\frac{\ln \gamma k}{(1-\gamma)}} \log_2 k \sum_{i = 0}^{\infty} \Brackets{\frac{2^{\frac{1}{(1-\gamma)}}}{e}}^i \notag\\
        &+ k 2^{\frac{\ln \gamma k}{(1-\gamma)}}\frac{1}{1-\gamma} \sum_{i = 0}^{\infty} i\Brackets{\frac{2^{\frac{1}{(1-\gamma)}}}{e}}^i \notag\\
        &+ k 2^{\frac{\ln \gamma k}{(1-\gamma)}}\frac{\ln \gamma k}{(1-\gamma)}\sum_{i = 0}^{\infty} \Brackets{\frac{2^{\frac{1}{(1-\gamma)}}}{e}}^i\notag\\
        &=k 2^{\frac{\ln \gamma k}{(1-\gamma)}} \log_2 k \frac{1}{1-r} \notag\\
        &+ k 2^{\frac{\ln \gamma k}{(1-\gamma)}}\frac{1}{1-\gamma} \frac{r}{(1-r)^2}\notag \\
        &+ k 2^{\frac{\ln \gamma k}{(1-\gamma)}}\frac{\ln \gamma k}{(1-\gamma)}\frac{1}{1-r}\notag\\
        &\in \OMath{k^{1 + 1/(\log_2 (1/\gamma)}\log k}.
        \label{eq_anytime_bound_t2}
    \end{align}

    Combining (\ref{eq_anytime_bound_t1}) and (\ref{eq_anytime_bound_t2}) in (\ref{eq_anytime_expected_t}) concludes the proof of \Cref{theo_runtime_self_adj_mr}.
    
\end{proofE}

%% file: sect_Experiments.tex
\section{Experimental Anytime Analysis of \BinVal}
\label{sect:experiments}
In this section, we present an experimental anytime analysis of the optimization of \BinVal, complementing the theoretical analysis provided in the previous sections.
We compare the anytime performances of the analyzed standard \EA with fixed mutation rate $1/n$ with the \EA with adjusting mutation rate that selects it greedy optimal, and the \EA with self-adjusting mutation rate.
In the following experiments, we average the results over 100 independent runs.
After a search for well-performing parameters for the \EA with self-adjusting mutation rate, we came up with~$a=1.85$ and $b=0.898$.

To perform the fixed-target analysis, we compare the number of iterations the algorithms require to optimize the first $k \leq n$ bits in the bit string that gets optimized using \BinVal as fitness function.
In Figure~\ref{exp:fig:optimized-bits-n-2048}, we show the number of iterations averaged over the 100 independent runs to optimize the $k \leq n$ leading bits for $n = 2048$.
As already indicated by our theoretical analysis in Sections~\ref{sect:ea-adjusting} and~\ref{sect:ea-self-adj}, the \EA with adjusting mutation rate and the~\EA with self-adjusting mutation rate clearly outperform the standard \EA for $k \in o(n)$.
This can be explained using \cite[Theorem~3.1]{WITT_2013} and the properties of \BinVal, since the two variants with adjusting mutation rate are able, in comparison to the \EA with fixed $\chi = \frac{1}{n}$, to use a mutation rate closer to the optimal one in the iterations until the first $k$ bits are set to $1$.
However, for $k > \frac{n}{2}$, the standard \EA is more efficient to optimize the first $k$ bits.

An extended experimental comparison of the \EA and its self-adjusting variant can be found in the appendix.

\begin{figure}[htp]
    \centering
    \includegraphics[width=\columnwidth]{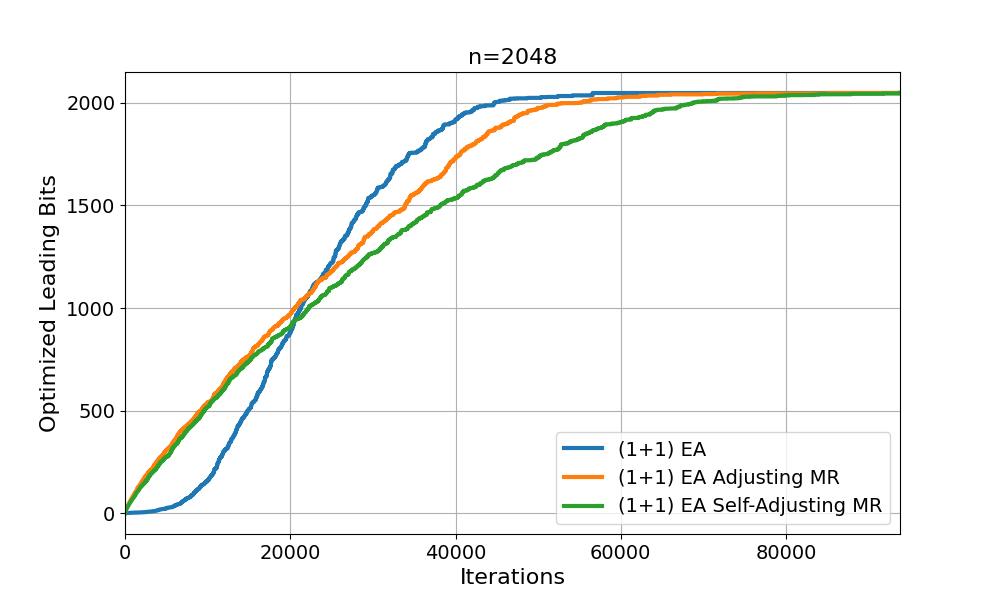}
    \caption{Comparison of the standard \EA, the \EA with adjusting mutation rate and the \EA with self-adjusting mutation rate. The number of iterations to optimize the first $k$ leading bits.}
    \label{exp:fig:optimized-bits-n-2048}
    \Description[\EA with adjusting and self-adjusting mutation rate outperforms \EA for $k\in o(n)$.]{Smaller number of iterations to optimize the first $k\in o(n)$ for the \EA with adjusting and self-adjusting mutation rate in comparison to the standard \EA.}
\end{figure}

%% file: sect_Conclusion.tex
\section{Conclusion}
\label{sect:conclusion}

Fixed-target analysis has both practical and theoretical relevance, delivering more insights regarding the anytime performance of algorithms.
We investigated this problem in this paper for \BinVal as the fitness function of several heuristic search algorithms.

With our theoretical analysis, we improved and extended the fixed-target run time bounds of the standard \EA and sig-cGA.
Furthermore, we considered a variant of the \EA with adjusting mutation rate and showed that, even for the self-adjusted handling of the mutation rate without further knowledge, the run time of the fixed-target analysis is independent of $n$ for all~$k~\in~o(n)$ simultaneously.
Our conjecture, based on the experimental investigations for realistic input sizes, is that the \EA with self-adjusting mutation rate optimizes the first $k\in o(n)$ bits in $\Theta(k\log k)$.

An open challenge for the fixed-target analysis of randomized search heuristics is to extend our results to general linear functions and compare their anytime performance.

%% file: sect_ExperimentsAppendix.tex
\section{Experimental Anytime Analysis of \BinVal}
\label{sec:appendix:expAnyAna}
To extend the experiments in Section \ref{sect:experiments}, we consider again the fixed-target run time of the standard \EA with fixed and the \EA with self-adjusting mutation rate, but execute both algorithms for different $n = 2^m$ with $m \in \CurlyBrackets{8,...12}$.
The results are again averaged over 100 independent runs.
In this setting, we first evaluate the maximum number of leading bits that the \EA with self-adjusting mutation rate optimizes with fewer iterations than the standard \EA with fixed mutation rate.
As shown in Figure~\ref{exp:fig:turning-points}, this number of leading bits is in relation to $n$ constant up to sampling variance (in the range of $[0.58, 0.61]$ for $n \in [256, 4096]$) for different values of $n$ as shown in Figure~\ref{exp:fig:turning-points}.
The results indicate that, independent of $n$, the \EA with self-adjusting mutation rate optimizes on average the first $k$ bits faster than the standard \EA for all $k \leq 0.6 n$.
Note that Figure~\ref{exp:fig:optimized-bits-n-2048} shows that for $n=2048$ the plots of the average number of iterations to optimize the first $k$ bits of the \EA and its self-adjusting variant intersect roughly at $k=0.5n$.
The difference to the average turning point at $k = 0.6 n$ as shown in Figure~\ref{exp:fig:turning-points} can be explained by the fact that the relative turning points are averaged over all 100 independent runs, while Figure~\ref{exp:fig:optimized-bits-n-2048} shows the average performance of the two algorithms.

\begin{figure}[H]
    \centering
    \includegraphics[width=\columnwidth]{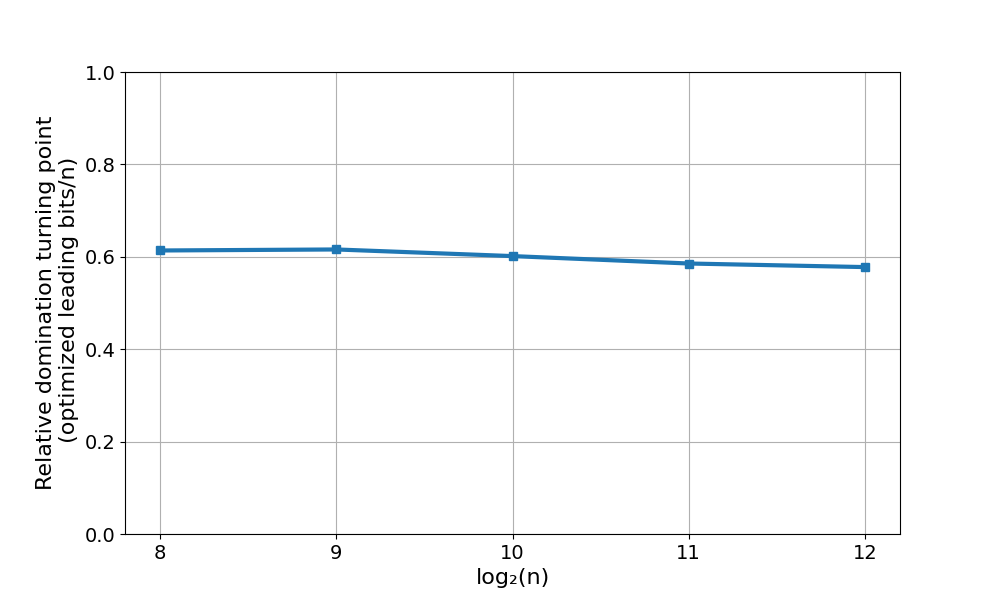}
    \caption{The number of leading bits in relation to $n$ until the \EA with self-adjusting mutation rate requires fewer iterations to optimize them than the standard \EA. Note that the $x$-axis is logarithmically scaled.}
    \label{exp:fig:turning-points}
    \Description[Constant relative turning point of \EA and \EA with self-adjusting mutation rate.]{For an increasing number of $n$, the relative number of leading bits that the \EA with self-adjusting mutation rate optimizes faster than the \EA stays constant.}
\end{figure}

In the same setting, we evaluate the average number of iterations to optimize the first $k \in \CurlyBrackets{16,32,64,128}$ bits.
Figure~\ref{exp:fig:iterations-until-k} shows the results for the standard \EA and \EA with self-adjusting mutation rate.
For the \EA with self-adjusting mutation rate, the number of iterations to optimize the first $k$ bits stays constant for a fixed~$k$ and larger values of $n$ and seems therefore independent of~$n$.
In comparison, the run time of the standard \EA significantly increases for a fixed $k$ and larger $n$. 
The strength of the \EA with self-adjusting mutation rate can be explained using the theoretical analysis in \Cref{sect:ea-self-adj}, which proves the independence of $n$ for the fixed-target analysis of $k \in o(n)$.

\begin{figure}[H]
    \centering
    \includegraphics[width=\columnwidth]{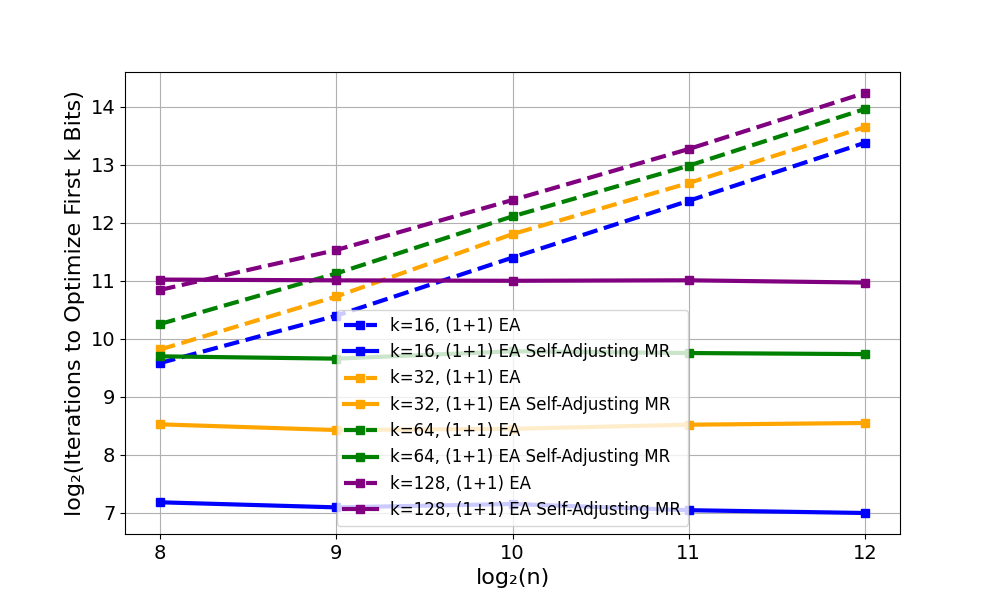}
    \caption{Comparison of the standard \EA (dashed lines) and the \EA with self-adjusting mutation rate (solid lines). For different $k$, we record the number of iterations until the first $k$ bits are optimized for different $n$. Each line represents a combination of a $k$ and an algorithm. Note that both axes are logarithmically scaled.}
    \Description[\EA with self-adjusting mutation rate optimizes first $k\in o(n)$ independent of $n$.]{Comparison for fixed values of $k \in o(n)$. \EA with self-adjusting mutation rate optimizes them independently of $n$, for \EA it clearly depends on $n$.}
    \label{exp:fig:iterations-until-k}
\end{figure}